\documentclass[twoside]{article}

\usepackage[accepted]{aistats2026}
\usepackage{natbib}
\usepackage[utf8]{inputenc}
\usepackage[T1]{fontenc}
\usepackage{hyperref}     
\usepackage{url}
\usepackage{booktabs}
\usepackage{amsfonts}
\usepackage{nicefrac}
\usepackage{microtype}
\usepackage{xcolor}

\usepackage{amsthm}
\usepackage{amsmath}
\usepackage{amssymb}
\usepackage{mathrsfs}
\usepackage[]{graphicx}

\usepackage{zref-clever}

\let\cref\zcref
\let\Cref\zCref

\zcsetup{
  cap        = true,   
  abbrev     = false   
}

\AddToHook{env/theorem/begin}{\zcsetup{countertype={theorem=theorem}}}
\AddToHook{env/prop/begin}{\zcsetup{countertype={theorem=proposition}}}
\AddToHook{env/lemma/begin}{\zcsetup{countertype={theorem=lemma}}}
\AddToHook{env/cor/begin}{\zcsetup{countertype={theorem=corollary}}}
\AddToHook{env/question/begin}{\zcsetup{countertype={theorem=remark}}} 
\AddToHook{env/hyp/begin}{\zcsetup{countertype={theorem=hypothesis}}}
\AddToHook{env/definition/begin}{\zcsetup{countertype={theorem=definition}}}
\AddToHook{env/remark/begin}{\zcsetup{countertype={theorem=remark}}}
\AddToHook{env/example/begin}{\zcsetup{countertype={theorem=example}}}
\AddToHook{env/concl/begin}{\zcsetup{countertype={theorem=result}}}      
\AddToHook{env/note/begin}{\zcsetup{countertype={theorem=note}}}
\AddToHook{env/conj/begin}{\zcsetup{countertype={theorem=proposition}}}  
\AddToHook{env/assumption/begin}{\zcsetup{countertype={theorem=theorem}}} 

\usepackage{comment}
\usepackage{algorithm}
\usepackage{algpseudocode}

\theoremstyle{plain}
\newtheorem{theorem}{Theorem}[section]
\newtheorem{prop}[theorem]{Proposition}
\newtheorem{lemma}[theorem]{Lemma}

\theoremstyle{definition}
\newtheorem{definition}[theorem]{Definition}
\newtheorem{remark}[theorem]{Remark}

\numberwithin{equation}{section}

\renewcommand{\paragraph}[1]{\noindent{\bf #1.}}
\usepackage{wrapfig}
\usepackage{caption}

\newcommand{\N}{\mathbb{N}}
\newcommand{\R}{\mathbb{R}}

\newcommand{\E}{\mathbb{E}}

\newcommand{\A}{\mathfrak{A}} 

\DeclareMathOperator{\diag}{\mathrm{diag}}

\DeclareMathOperator{\tr}{\mathrm{tr}}
\DeclareMathOperator{\Tr}{\mathrm{Tr}}
\DeclareMathOperator{\rank}{\mathrm{rank}}
\DeclareMathOperator{\cov}{\mathrm{Cov}}
\DeclareMathOperator{\softmax}{\mathrm{softmax}}

\newcommand{\eps}{\varepsilon}

\newcommand{\stnorm}{\mathcal{N}(0,1)}

\newcommand{\trunc}{K}

\newcommand{\Ylln}{Y}
\newcommand{\Ycntr}{Y^f}
\newcommand{\Ypoly}{Y^Q}
\newcommand{\Ypolylin}{Y_\mathrm{lin}^Q}

\newcommand{\Ylin}{Y^f_\mathrm{lin}}

\newcommand{\smat}{S}
\newcommand{\ptn}{Z}
\newcommand{\ntoken}{\ell}
\newcommand{\dimqk}{d_\mathrm{qk}}
\renewcommand{\Pr}{\mathbb{P}}

\begin{document}

%

%

\twocolumn[

\aistatstitle{Gaussian Equivalence for Self-Attention: Asymptotic Spectral Analysis of Attention Matrix}

\aistatsauthor{ Tomohiro Hayase \And Beno{\^\i}t Collins \And  Ryo Karakida }

\aistatsaddress{ AIST and CoeFont, Co., Ltd. \And  Kyoto University \And AIST and RIKEN AIP} ]

\begin{abstract}
Self-attention layers have become fundamental building blocks of modern deep neural networks, yet their theoretical understanding remains limited, particularly from the perspective of random matrix theory.
In this work, we provide a rigorous analysis of the bulk singular value distribution of the attention matrix and establish the first Gaussian equivalence result for attention. In a natural regime where the inverse temperature remains of constant order, we show that the singular value distribution of the attention matrix is asymptotically characterized by a tractable linear model.
We further demonstrate that the distribution of squared singular values deviates from the Marchenko-Pastur law, which has been believed in previous work.
Our proof relies on two key ingredients: precise control of fluctuations in the normalization term and a refined linearization that leverages favorable Taylor expansions of the exponential.
This analysis also suggests a threshold for linearization and elucidates why attention, despite not being an entrywise operation, admits a rigorous Gaussian equivalence in this regime.
\end{abstract}

\section{Introduction}
Self-attention, popularized by the Transformer \citep{vaswani2017attention}, has become a standard component across modalities. Yet, compared with feedforward layers, the asymptotic behavior of self-attention remains less understood from the viewpoint of random matrix theory (RMT). For standard feedforward architectures, RMT has already provided a powerful toolkit. \citet{pennington2017nonlinear} initiated nonlinear random matrix theory for deep learning by deriving the spectral distribution of random-feature Gram matrices, opening the door to precise spectral analyses of neural kernels \citep{fan2020spectra}, universality results for high-dimensional learning \citep{hastie2022surprises,montanari2022universality,goldt2022gaussian}, and spectral characterizations of networks updated by finitely many training steps \citep{wang2024nonlinear}. On the signal-propagation side, RMT and free probability have characterized Jacobian spectra and dynamical isometry at initialization \citep{pennington2017resurrecting, pennington2018emergence} and clarified Fisher-information spectra in wide random networks \citep{karakida2019universal, amari2019fisher, hayase2021spectrum}.
High-dimensional limits often show universality, enabling tractable deterministic equivalents that capture limiting spectra and performance.

Recent theory has analyzed attention under several complementary asymptotic regimes. Infinite-head limits recover GP/NTK descriptions for multi-head attention \citep{hron2020infinite}; tensor-program analyses identify the exact infinite-width law of a single standard attention layer \citep{sakai2025inf}; and infinite-size analyses study either training dynamics or signal propagation in deep transformers \citep{bordelon2024infinite, giorlandino2025two}. Our focus is complementary: we analyze a single standard softmax attention matrix at random initialization in a proportional high-dimensional regime where the context length grows with the embedding dimension and $\beta=O(1)$ remains fixed. In this regime, we obtain an exact baseline for the bulk non-Perron singular-value distribution. We do not study training dynamics here; rather, the present theorem is intended as a building block for later analyses of trainability, depth-wise signal propagation, and early feature learning.

A recurring principle in these previous works on usual feedforward architectures is Gaussian equivalence: high-dimensional nonlinear models often share their asymptotic behavior with Gaussian models determined by a few statistics \citep{pennington2017nonlinear}. 
The attention structure, however, brings a distinct difficulty. The softmax over pairwise scores produces a full matrix with row-wise normalization, coupling all tokens and breaking i.i.d.\ structure. Standard tools for entrywise nonlinearities do not directly apply.

We prove the first Gaussian equivalence for self-attention at random initialization in the fixed softmax-temperature regime, yielding a tractable linear model that asymptotically characterizes the bulk singular-value distribution of the attention matrix (\cref{theorem:main}). Our proof sharply controls fluctuations of the softmax normalizers and approximates the exponential by Taylor polynomials of growing degree and applies a linearization argument that preserves the bilinear score structure. This suggests a candidate breakdown scale where Gaussian equivalence fails and explains why attention, despite not acting entrywise, admits a rigorous Gaussian approximation in this regime. The predicted limit for the squared singular values departs from the classical Marchenko-Pastur law (contrary to prior claims \citep{saada2025mind}), and we confirm the prediction numerically at a large dimension  (\cref{fig:empirical-hist-six} and \cref{ssec:empirical-six}). We expect this Gaussian equivalence to provide a foundation for the statistical analysis of self-attention layers.

The singular spectrum is relevant because recent theory connects spectral or signal-propagation properties of attention to concrete failure modes of transformers. In particular, the gap between leading singular values has been identified as a mechanism for rank collapse in width \citep{saada2025mind}, while signal-propagation analyses at initialization link the scale of attention to rank collapse, entropy collapse, gradient imbalance, and trainability in deep transformers \citep{noci2022signal, giorlandino2025two}. Our result should therefore be read as a precise baseline for the bulk spectrum of standard softmax attention before repeated depth composition or optimization modifies the layer.
Outliers beyond the Perron mode and long-context scalings beyond $\beta=O(1)$ are not resolved here and are discussed only heuristically in \cref{ssec:discussion-outliers,ssec:discussion-beta}.

\subsection{Related Work}

\paragraph{Complementary asymptotic theories of attention}
Prior theoretical analyses of attention mostly operate in regimes different from ours. Infinite-head limits recover GP/NTK descriptions for multi-head attention \citep{hron2020infinite}. For single standard attention layers, \citet{sakai2025inf} identify the exact infinite-width law via tensor programs, again in a different limit from the proportional regime considered here. At the level of deep transformers, \citet{bordelon2024infinite} analyze infinite limits of training dynamics, whereas \citet{giorlandino2025two} study signal propagation at initialization and characterize rank-collapse and entropy-collapse regimes. Our proportional-limit result is complementary to these works: it gives an exact random-initialization baseline for the bulk non-Perron singular-value distribution of a single standard softmax attention matrix when the context length and embedding dimension grow together.

\paragraph{Spectral analysis of attention matrices}
Prior work~\citep{saada2025mind} modeled first-layer key-query attention with orthonormal inputs by a random Markov surrogate generated from i.i.d.\ weights followed by row normalization; for this surrogate, the bulk singular-value law is quartercircular~\citep{bordenave2012circular}, equivalently yielding a Marchenko--Pastur law for squared singular values. Our result instead treats the exact bilinear softmax attention matrix in a proportional regime and shows that its limiting bulk law differs from that surrogate prediction. 
A recent analysis~\citep{liao2025random} studies an entrywise attention variant under signal-plus-noise inputs and a full-plus-rank-one structure in the attention weights ($W_K^\top W_Q = I + w_K w_Q^\top$), and remarks that the normalizer is asymptotically constant for truncated-exponential attention up to a scaling factor. Our control of fluctuations of the normalization term provides related support for analogous normalizer replacements and may strengthen such entrywise approximations. Our results, however, target the bulk singular-value distribution of standard softmax attention with general bilinear scores.

\paragraph{Nonlinear RMT, Gaussian equivalence, and early feature learning}
Beyond GP/NTK limits, a substantial random-matrix literature studies nonlinear random-feature and kernel matrices directly. \citet{pennington2017nonlinear} initiated nonlinear random matrix theory for deep learning, \citet{louart2018random} derived deterministic equivalents for Gram matrices of random neural feature maps, and \citet{benigni2021eigenvalue} established rigorous bulk laws for nonlinear models. Gaussian equivalence has since been extended to structured-data settings \citep{goldt2022gaussian}. In a related direction, \citet{wang2024nonlinear} study nonlinear spiked covariance models and signal propagation through trained features, providing spectral characterizations that go beyond static kernels. Closest to the present motivation, recent one-step feature-learning analyses show that even a very small amount of training can create low-rank spikes or spiked random-feature equivalents, thereby changing both spectra and generalization \citep{ba2022highdim, cui2024asymptotics, moniri2024theory}. Our contribution extends this nonlinear RMT / Gaussian-equivalence program to standard softmax attention, where the softmax normalizer couples entries across each row and the score matrix itself is not i.i.d.

A related but distinct line uses free probability to analyze deep-network operators rather than random-feature Gram matrices.
\citet{collins2023freeness} established asymptotic freeness of layerwise Jacobians in Haar-orthogonal MLPs, providing a rigorous basis for spectral propagation analyses.
At a more matrix-theoretic level, entrywise nonlinear transforms have also been studied from free-probabilistic and combinatorial viewpoints \citep{dabo2024traffic,speicher2024entrywise}.
Unlike these settings, attention is not an entrywise transform; we show that a softmax-type normalization nonetheless admits a Gaussian equivalent.

\paragraph{Scale of inverse temperature}
Most prior work has theoretically analyzed self-attention in infinite limits with a fixed context length, such as infinite embedding dimension or infinitely many heads \citep{hron2020infinite,bordelon2024infinite,sakai2025inf}. In contrast, we consider a proportional limit where the context length grows with the embedding dimension, allowing for longer contexts, which is a more realistic setting. Empirical and practical work suggests that the effective attention scale matters in long contexts: \citet{zhou2025length} analyze length-induced embedding collapse and propose TempScale as mitigation, and \citet{nakanishi2025scalable} empirically motivate a context-dependent scaling in SSMax. On the theoretical side, \citet{giorlandino2025two} analyze a REM-motivated $\Theta(\sqrt{\log \ntoken})$ scaling in a complementary infinite-sequence-length setting and identify a critical parameter separating rank-collapse and entropy-collapse regimes. At a more task-specific level, \citet{chiang2022overcoming} study single-position focusing and length generalization under explicit temperature adjustment. Our analysis is restricted to $\beta=O(1)$; Section~6 provides heuristic evidence that $\sqrt{\log \ntoken}$ may mark the breakdown of the present linearization, but a rigorous treatment of this scale remains open.
This long-context perspective is also consistent with practical proposals that explicitly modify or rescale attention in order to preserve selectivity at large context lengths, including LogN-Scaling as used in Qwen \citep{bai2023qwen}, YaRN \citep{peng2023yarn}, and SWAN-GPT \citep{puvvada2025swangpt}. On the theoretical side, \citet{chen2026critical} prove a sharp phase transition in a simplified long-context attention model under explicit attention scaling. Their regime and critical law are different from ours, but together these results support the view that the scale of attention is a first-order issue in long-context transformers.

\section{Preliminaries}
Let $X \in \R^{\ntoken \times d}$ be an input sequence of context length $\ntoken \in \N$ with $d \in \N$.
Let $W^Q, W^K$ be $d \times \dimqk$ random matrices with entries independently sampled from $\mathcal{N}(0,1)$  such that the collection of their entries forms an independent family.  
Given the input matrix $X$, we define the score matrix $\smat \in \R^{\ntoken \times \ntoken}$ by
\begin{align}
    \smat = \frac{1}{\sqrt{\dimqk}}\, X W^Q (W^K)^\top X^\top.\label{align:smat}
\end{align}

We consider the same assumption on $X$ as \citep{saada2025mind}:  $X$ is deterministic with  $\ntoken \leq d$ and $XX^\top=I_{\ntoken}$. 
\begin{remark}
   The difference in scaling conventions is immaterial. Specifically, one may either adopt  
   (1) $\|X_{i.:}\|^2 = 1$ with $W_{ij} \sim \stnorm$, or  
   (2) $\|X_{i.:}\|^2/d =1$ with $W_{ij} \sim \mathcal{N}(0,1/d)$.  
   While (2) is the more practical choice, it obscures the order of individual entries in the proof. For clarity and consistency of exposition, we employ the convention in (1) throughout this work, which is compatible with $XX^\top=I_\ntoken$.
\end{remark}

Fix $\beta > 0$.  
We then define the $\ntoken \times \ntoken$ matrix $A$ by
\begin{align}\label{align:def-A}
    A_{ij} =\ \frac{\exp\!\bigl(\beta S_{ij}\bigr)}{\ptn_i},
\end{align}
where
\begin{align}\label{align:z}
    \ptn_i =\sum_{j=1}^{\ntoken} \exp\bigl(\beta S_{ij}\bigr).
\end{align}
In particular, we write $A=\softmax(\beta S)$.

\begin{remark}
Let $a,b,x,y \in \R^{\ntoken}$ and write $\xi = a^\top W^Q (W^K)^\top b$, and $\zeta = x^\top W^Q (W^K)^\top y$.    
Then a direct computation yields
 $\cov(\xi^{2},\zeta^2)/2d
    = \|a\|_2^{2}\|x\|_2^{2}\,(b^\top y)^{2}
    + (a^\top x)^{2}\,\|b\|_2^{2}\|y\|_2^{2}  
    + (d+1)\,(a^\top x)^{2}(b^\top y)^{2}.$
Thus, the entries of $X W^Q (W^K)^\top X^\top$ are \textbf{not} independent.  Therefore, one can not directly apply the circular law theorem for Markov random matrices by \citep{bordenave2012circular}.
\end{remark}

\begin{definition}
    For any $d \in \N$ and  real matrix $Y \in \R^{d \times d}$, we denote its singular values (i.e.\ the square roots of the eigenvalues of $YY^\top$) by $
    s_1(Y) \geq s_2(Y) \geq \dots \geq s_d(Y)$.
    The empirical \textbf{squared} singular value distribution $\nu_Y$ of $Y$ is defined as
    \begin{align}
        \nu_Y = \mu_{YY^\top}
        = \frac{1}{d}\sum_{i=1}^d \delta_{s_i(Y)^2},
    \end{align}
    where $\delta_s$ denotes the Dirac measure at $s \in \R$.
\end{definition}

\begin{remark}
    For each $q \in \N$, the $q$-th moment of $\nu_Y$ is given by
    \begin{align}
       m_q(YY^\top)=  \tr\![(YY^\top)^{q}]=m_q(\nu_Y),
    \end{align}
    where $\tr = d^{-1}\Tr$ denotes the normalized trace.
\end{remark}

\section{Gaussian Equivalence}

\begin{table*}[t]
    \centering
    \caption{Equivalent Random matrix models, in the sense of sharing the asymptotic singular value distributions.}
    \label{tab:equiv-models}
    \begin{tabular}{c|ccc}
    \toprule
    Notation & Definition & Key Parameters & Eq.\\
    \midrule
      $\sqrt{d}A$   &  $\sqrt{d}\softmax(\beta S)$   &  $\beta>0$ (constant). & \eqref{align:def-A}\\           
            $\sqrt{d}A^{\perp}$   &  $\sqrt{d}(A-u_du_d^\top)$  &  $u_d=(1,1,\dots,1)/\sqrt{d}$ & \eqref{align:aperp}\\
      $\Ylln$   &   $\exp(\beta S)/ (e^{\beta^2/2}\sqrt{d})$  & -  & \eqref{align:ylln} \\
      $\Ycntr$ &   $f(S)/\sqrt{d}$  & $f(x)=\exp(\beta x - \beta^2/2)-1$.& \eqref{align:yllnpf}\\
      $\Ypoly$  & $Q_{n_d}(S)/\sqrt{d}$  & $n_d=\lceil c \log d / \log\log d \rceil$. & \eqref{align:ypoly}\\
      $\Ypolylin$    & $\sqrt{\theta_2^Q}S/\sqrt{d} + \sqrt{\theta_1^Q-\theta_2^Q} W/\sqrt{d}$  & $\theta_i^{Q}= \theta_i(Q_{n_d})$ ($i=1,2$). & \eqref{align:ypolylin}\\
      $\Ylin$    & $\sqrt{\theta_2}S/\sqrt{d} + \sqrt{\theta_1-\theta_2}W/\sqrt{d}$  & $\theta_1=e^{\beta^2}-1$, $\theta_2=\beta^2$. & \eqref{align:ylin}\\
    \bottomrule
    \end{tabular}
\end{table*}
\begin{figure*}[th]
    \centering
    \includegraphics[width=\linewidth]{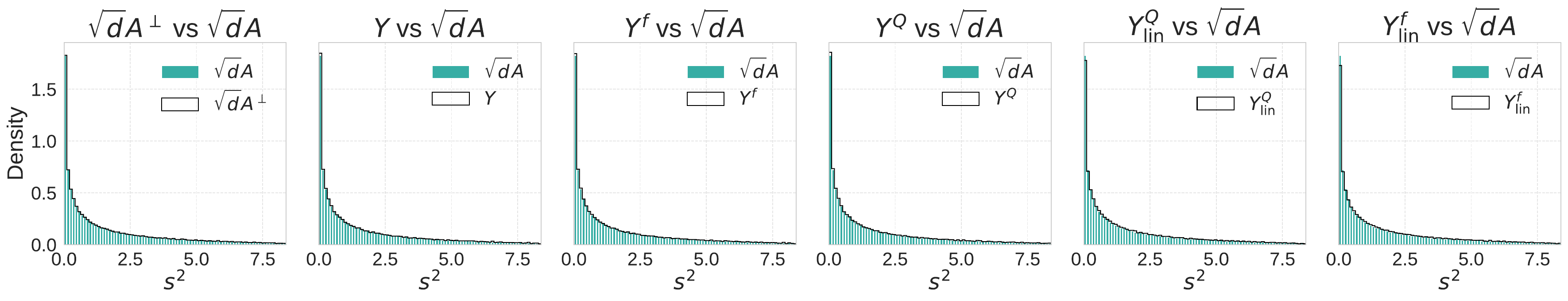}
    \includegraphics[width=\linewidth]{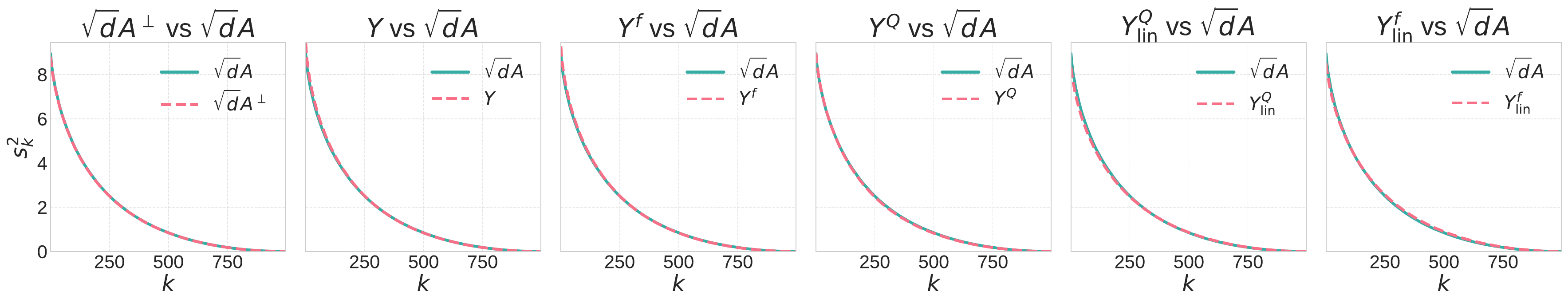}
\caption{\textbf{Stepwise approximation of attention leading to the Gaussian Equivalence \(Y^f_{\mathrm{lin}}\).} Top: histograms of empirical squared singular values; bottom: sorted \(s_k^2\). Each panel compares the scaled attention matrix \(\sqrt{d}A\) with a model introduced in the proof, shown left to right in the order \(\sqrt{d}A^\perp, Y , Y^{f}, \Ypoly, \Ypolylin,\Ylin\). The bulk spectra (top three singular values removed; see \cref{fig:topk}) are nearly indistinguishable, showing that the $\sqrt{d}A$ is accurately approximated throughout and culminates in \(Y^{f}_{\mathrm{lin}}\). Settings: \(d=d_{qk}=1000\), \(\beta=1\), 10 draws. 
}
    \label{fig:empirical-hist-six}
\end{figure*}

Fix $\beta>0$ and define
\begin{align}\label{align:f}
f(x)=\exp\!\bigl(\beta x-\tfrac{\beta^{2}}{2}\bigr)-1,    
\end{align}
so that $\E[f(\chi)]=0$ when $\chi\sim\stnorm$. 
Throughout this paper, we denote by $\chi$ a random variable distributed as $\stnorm$, i.e.\ the standard normal distribution.
For any function $g$,  define
\begin{align}
    \theta_1(g) = \E[g(\chi)^2], \quad
    \theta_2(g) = \E[g^\prime(\chi)]^2,
\end{align}
%
if the derivative and the expectations are well-defined. Set
\begin{align}
\theta_{1}=\theta_1(f)=e^{\beta^{2}}-1, \quad 
\theta_{2}=\theta_2(f)=\beta^{2}.
\end{align}
Note that $\theta_{1}\ge \theta_{2}$ by $e^{x}\ge 1+x$.
Let $W$ be $\ntoken \times \ntoken$ random matrix of $\stnorm$ entries, with the family of all entries of $(W^Q, W^K, W)$ being independent. Define
\begin{align}\label{align:ylin}
    \Ylin= \sqrt{\theta_{2}} S/\sqrt{\ntoken}+ \sqrt{\theta_{1}-\theta_{2}}W/\sqrt{\ntoken},
\end{align}
where $S$ is given by \eqref{align:smat}.

We work in the proportional asymptotic regime
\begin{align}\label{align:limit}
    d,\ntoken,\dimqk \to \infty,
\end{align}
such that
\begin{align}\label{align:limit-ratio}
    \frac{\ntoken}{d} \to \gamma,
    \qquad
    \frac{\dimqk}{d} \to \psi,
\end{align}
for some constants $\gamma,\psi>0$.
Then, by the asymptotic freeness of $(W^Q,W^K,W)$, one has almost surely
\begin{align}
    \nu_{\Ylin}
    \xrightarrow{\mathrm{moments}}
    \nu_\infty(\beta,\gamma,\psi)
\end{align}
as \eqref{align:limit} with \eqref{align:limit-ratio}, where $\nu_\infty(\beta,\gamma,\psi)$ is a deterministic compactly supported probability measure on $\R_{\ge 0}$ depending only on $\beta,\gamma,\psi$.

\begin{remark}
Unless otherwise stated, we suppress the dependence on the parameters $\beta$, $\gamma$, and $\psi$ in the notation, and write $\nu_\infty(\beta,\gamma,\psi)$ simply as $\nu_\infty$.
The limiting distribution $\nu_\infty$ is characterized by a fixed-point equation for its Cauchy--Stieltjes transform $G_{\nu_\infty}(z)$, which agrees with the formulation obtained in \citep{pennington2017nonlinear}. For a rigorous proof, we refer to \citep{benigni2021eigenvalue}.
\end{remark}

 Now we have prepared to describe the main theorem.
\begin{theorem}\label{theorem:main}
Let $A\in\R^{\ntoken\times \ntoken}$ be the attention matrix from~\eqref{align:def-A}.  
Then almost surely it holds that
\begin{align}
s_1(A)=\|AA^\top\|_\infty^{1/2} \to 1,
\end{align}
\begin{align}
\nu_{\sqrt{\ntoken}A^\perp} \xrightarrow{\mathrm{moments}} \nu_\infty,
\end{align}
where 
\begin{align}\label{align:aperp}
A^\perp = A -u_{\ntoken} u_{\ntoken}^\top, 
\end{align}
with $u_{\ntoken}=(1,\dots, 1)^\top/\sqrt{\ntoken}$, and,
\begin{align}
\nu_{\sqrt{\ntoken}A} \xrightarrow{\mathrm{weakly}} \nu_\infty(\beta, \gamma, \psi),
\end{align}
with \eqref{align:limit} and \eqref{align:limit-ratio}.
\end{theorem}
The proof of \cref{theorem:main} is deferred to \cref{ssec:proof-of-thm}. 

In particular, \cref{theorem:main} also fixes the bulk scale of the spectrum: after removing the Perron direction via the rank‑one projection $u_{\ntoken} u_{\ntoken}^\top$, the empirical distribution of the singular values of $\sqrt{\ntoken}(A - u_{\ntoken} u_{\ntoken}^\top)$ converges almost surely to a nondegenerate limit $\nu_\infty$. Consequently, apart from the largest singular value (which tends to $1$), almost all singular values of $A$ have magnitudes on the order of $d^{-1/2}$. This delineates a clean rank‑one plus diffusive bulk structure for attention: one macroscopic mode at scale $1$ and a sea of modes at scale $d^{-1/2}$.

\begin{remark}[Reductions]\label{rem:X-Id}
By applying a $d\times d$ $\rank$-$\ntoken$ orthogonal projection to both models $\sqrt{\ntoken}A$ and $\Ylin$, the case $\ntoken<d$ is obtained as a projected version of the case $\ntoken=d$. Hence, it suffices to prove the result for $\ntoken=d$.
Next, by left-orthogonal invariance of the Ginibre ensemble and the independence of $(W^{Q}, W^{K},W)$, the joint law of
\(
(XW^{Q},\,X W^{K},\,W)
\)
is the same as that of
\(
( W^{Q},\, W^{K},\,W).
\)
Hence, the distributional statements below are unchanged if one replaces $X$ with $I_d$. We, therefore, work without loss of generality with $X=I_d$.  
Lastly, since  we take the proportional limit~\eqref{align:limit-ratio} with $\psi \in (0,+\infty)$, the rectangular shape of each $W^Q$ and $W^K$ does not cause a significant change in the proof. Thus, we also assume that $\dimqk=d$ to avoid complicated notation.  In summary, later in this section, we assume $\ntoken=d=\dimqk$ and $X=I_d$.
\end{remark}

\section{Equivalent Random Matrix Models}
We next explain how \cref{theorem:main} is established. The proof proceeds by a sequence of transformations, each of which preserves the asymptotic singular value distribution. Along the way, these transformations also suggest the underlying reason why the theorem holds. \cref{tab:equiv-models} summarizes the sequence of random matrices, all of which share the same limiting singular value distribution.  

Intuitively, the argument relies on the representation
\begin{align}
    \sqrt{d}A = \left[ \diag\left(\frac{Z}{\E[e^{\beta \chi}]d}\right)\right]^{-1}
    \frac{\exp(\beta S)}{\E[e^{\beta \chi}]\sqrt{d}},
\end{align}
where the normalizers $Z_i/d$ are first replaced by deterministic constants, after which the exponential function is expanded via a Taylor series to obtain a linear approximation. The proof consists of making each of these steps rigorous.  

The central task is to ensure the simultaneous validity of three bounds:  
(i) the tail estimate for $|S_{ij}|$ with truncation parameter $\trunc = c\log d / \log \log d$ or $(\log d)^{1/2+\delta}$,  
(ii) the fluctuations of the normalizer, controlled at order $d^{-1/2+\delta}$, and  
(iii) the polynomial approximation degree $n_d = \lceil c \log d / \log \log d \rceil$.  
Showing that these bounds are compatible constitutes the core of the overall strategy.

\subsection{Fluctuation of Normalizer}
In this section, we replace the normalization term with a deterministic constant.  
Let $X_1,\dots,X_d\stackrel{\mathrm{iid}}{\sim}\mathcal N(0,1)$. The \emph{law of the iterated logarithm} (LIL; \citep[Theorem~8.5.1]{durrett2019probability}) determines the fluctuations of the sum;
\begin{align}\label{align:lil}
\limsup_d \frac{\bigl|X_1+\cdots+X_d\bigr|}{\sqrt{2d\log\log d}}=1, \quad \text{a.s.}
\end{align}
Since $d^{\delta} > \sqrt{2\log\log d}$ for any $\delta>0$ and  sufficiently large $d$,
the following holds: 
\begin{align}
\lim_d  d^{1/2-\delta}\frac{X_1+\cdots+X_d}{d}=0, \quad\text{a.s.}
\end{align}
The same conclusion with $\delta$ holds for the non-independent arrays after replacing $X_j$ by $S_{ij}-\E[S_{ij}]$ from \cref{lem:as-of-dominant} without LIL. Precisely,
we obtain concentration of $Z_i$ around the constant mean as follows: for fixed $\delta \in (0,1/2]$, it holds a.s.\,as $d \to \infty$,
\begin{align}\label{align:bounds-column-sum}
d^{1/2-\delta} \max_{i=1,2,\dots,d}  \Big| d^{-1}Z_i - \E[e^{\beta\chi}] \Big| \to 0.
\end{align}
Therefore, we define the approximation obtained by applying concentration:
\begin{align}\label{align:ylln}
    \Ylln = \frac{\exp(\beta S)}{ e^{\beta^2/2} \sqrt{d}},
\end{align}
where we have used  $\E[e^{\beta \chi}] = e^{\beta^2/2}$.  

\begin{lemma}[Concentration]\label{lem:main-concentration}  
Fix $\delta \in (0,1/2]$. Then, almost surely, there exists a sequence $(\varepsilon_d)_d$ with $\lim_d d^{1/2-\delta}\varepsilon_d=0$ such that
\begin{align}
   \lim_{d\to\infty} \max_{i=1,2,\dots,d} \bigl|\, s_i(\sqrt{d}A) - (1+\varepsilon_d)s_i(\Ylln) \,\bigr| = 0.    
\end{align}
\end{lemma}
\begin{proof}    
Define $D=\diag(\ptn_1, \dots, \ptn_d)/( \E[e^{\beta\chi}]d)$. Then 
\begin{align}
   \sqrt{d}\,A = D^{-1} Y.
\end{align}
By \cref{lem:as-of-dominant}, we obtain the concentration of $Z_i$ around the constant mean as \eqref{align:bounds-column-sum}.
By the diagonality of $D$, it holds that
\begin{align}
\frac{s_i(Y)}{\max D}  \leq s_i(\sqrt{d}A) \leq \frac{s_i(Y)}{\min D}.
\end{align}
From \eqref{align:bounds-column-sum}, 
$|1/D_{ii} - 1|= | 1 - D_{ii}|/|D_{ii}|
   = o(d^{-1/2+\delta})/(1+o(d^{-1/2+\delta}))
   = o(d^{-1/2+\delta}).$
Thus $1/\min D, 1/\max D = 1+o(d^{-1/2+\delta})$ as $d\to\infty$ a.s., which proves the assertion.
\end{proof}

From \cref{lem:main-concentration}, we have the estimation of $s_1(A)$ without the detail of $\eps_d$ in \cref{prop:s1}. We use the order of $\eps_d$ in \cref{ssec:proof-of-thm}.
\begin{prop}\label{prop:s1}
    $\lim_d s_1(A)=\lim_d s_1(Y/\sqrt{d})=1$, a.s.
\end{prop}
\begin{proof}
By the symmetry of $S$, the same bound holds in the column direction as \eqref{align:bounds-column-sum} and thus
\begin{align}
   \max_{j=1,2,\dots,d}\Big|\frac{1}{\sqrt{d}}\sum_{i=1}^d Y_{ij} - 1\Big|\to 0,
   \quad \text{a.s.}    
\end{align}
Since 
\begin{align}    
   \|YY^\top\|_\infty 
   \leq \Big(\max_{j}\sum_{i}Y_{ij}\Big)\Big(\max_{i}\sum_{j}Y_{ij}\Big),
\end{align}
we obtain $\limsup_{d}\|Y/\sqrt{d}\|_\infty \leq 1$.  

On the other hand, as $(Y/d)u_d = (1+o(1))u_d$, we also have $\liminf_{d}\|Y/\sqrt{d}\|_\infty \geq 1$.  
In summary, $\lim_{d \to \infty}s_1(Y/\sqrt{d})=1$ a.s.  
Finally, by \cref{lem:main-concentration}, we conclude that $\lim_{d}s_1(A) = 1$ a.s.
\end{proof}

\subsection{Rank-One Perturbation and Interlacing}

To apply linearization, we consider the function $f$ defined in \eqref{align:f} by normalizing $x \mapsto \exp(\beta x)$ so that $\E[f(\chi)] = 0$.  
Define the corresponding matrix by
\begin{align}\label{align:yllnpf}
   \Ycntr = \frac{1}{\sqrt{d}} f(S).
\end{align}
In this formulation, $\Ycntr$ serves as an approximation of $\Ylln$.  
Subtracting the mean corresponds exactly to removing a rank-one projection from $\Ylln$: the two matrices differ only by a perturbation aligned with the all-ones vector.  
Since such rank-one updates affect singular values only in a tightly controlled manner, the spectra of $\Ylln$ and $\Ycntr$ must remain close, a fact formalized by the interlacing theorem below.

\begin{lemma}[Interlacing]\label{lem:main-interlacing}
For $i=1,\dots,d$, we have
\[
   s_{i+1}(\Ycntr) \leq s_i(\Ylln) \leq s_{i-1}(\Ycntr),
\]
under the convention $s_0 = +\infty$ and $s_{d+1}=0$.
\end{lemma}

\begin{proof}
Observe that $\Ycntr = \Ylln - \sqrt{d}\, u_d u_d^\top$, where $u_d = (1,1,\dots,1)^\top/\sqrt{d}$.  
Thus, $\Ylln$ differs from $\Ycntr$ by a rank-one perturbation.  
By the interlacing theorem for rank-one perturbations~\cite[Theorem~1]{thompson1976behavior}, the claim follows. 
\end{proof}

\subsection{Taylor Expansion and Strong Convergence}

We introduce the linearization $\Ylin$ of $\Ycntr$ following \citet{benigni2021eigenvalue} and using $\E[f(\chi)]=0$.  
The idea is to approximate $f$ by centered polynomials $Q_n$ with $\E[Q_n(\chi)]=0$; since linearization is proved for polynomial nonlinearities, a sufficiently accurate polynomial approximation transfers the result to $f$.

The essential component of $f$ is $e^{x}$, which does not satisfy the assumption~\citep[(2.4)]{benigni2021eigenvalue}.  
However, the assumption can be relaxed: it suffices to require the approximation only for \emph{large} degrees.  
Specifically, by \cref{lem:taylor-beta}, it holds that for any $K>0$ and all $n>K$,
\begin{align}
   \max_{\beta|x|\le K}\bigl|e^{\beta x}-P_n(\beta x)\bigr|
   \le \Bigl(\frac{eK}{n}\Bigr)^{\!n},
\end{align}
where $P_n(x)=\sum_{k=0}^{n-1}x^k/k!$.
This yields the entrywise approximation of $f(S)$ in \cref{lem:exp-approx-bound-beta}, on which the proof crucially relies.

Set
\begin{align}
    Q_n(x)=e^{-\beta^2/2}\Bigl(P_n(x)-\E\bigl[P_n(\chi)\bigr]\Bigr),
\end{align}
and define
\begin{align}\label{align:ypoly}
    \Ypoly=\frac{1}{\sqrt{d}}\,Q_{n_d}(S),
\end{align}
where
\begin{align}\label{align:nd}
    n_d=\bigl\lceil c\,\frac{\log d}{\log\log d}\bigr\rceil,
\end{align}
for a constant $c>1/(1-2\delta)$ with $\delta\in(0,1/2)$, and $\lceil\cdot\rceil$ denotes the ceiling function.

\newcommand{\good}{G_\trunc}
\begin{lemma}\label{lem:stong-conv}
$||\Ycntr -\Ypoly ||_\infty \to 0$ as $d \to \infty$, a.s.
\end{lemma}
\begin{proof}
    Let $\trunc>0$, then choose a good set  
    \begin{align}
        \good=\{ \max_{i,j}|S_{ij}| \leq  \trunc \}.
    \end{align}
 Then the concentration of $\chi^2$-distribution,
    \begin{align}
        P(\good^c) \leq  Cd^2\exp(-\trunc^2/2).
    \end{align}
    Thus, we set 
    \begin{align}\label{align:trunc-sqrt-logd}
        \trunc= (\log d)^{1/2+\delta},
    \end{align} 
    then $P(\good^c)$ is super-polynomial decay.
    Thus, on the good set $\good$,
    \begin{align}
        || \Ycntr - \Ypoly ||_\infty \leq \sqrt{d} \sup_{x \in [-\trunc, \trunc]} |f(x) - Q_n(x)| 
    \end{align}
    
    By \cref{lem:exp-approx-bound-beta}, there exist constants $c_1,c_2 > 0$ 
    \begin{align}\label{align:f-q}
        \sup_{x \in [-\trunc, \trunc]} |f(x) - Q_n(x)| \leq c_2 (c_1 \trunc/n)^n,
    \end{align}
    where $n \in \N$ with $n>\beta \trunc$ and sufficiently large $d$. 
    Then  $n/\trunc \geq  c (\log d)^{1/2-\delta}/\log\log d >1$.
    Now
        $n \log (n/\trunc) = c ((1/2 -\delta)+ o(1)) \log d $.
    With $c> 1/(1-2\delta)$, $ n \log (n/\trunc)>(1/2+\eps) \log d$ and $(\trunc/n)^n=o(d^{-1/2})$. 
    Then the RHS of \eqref{align:f-q} is $o(1/\sqrt{d})$ and it proves the assertion.
\end{proof}

\subsection{Gaussian Equivalence for Polynomials}

For $n\in\mathbb{N}$, define
\begin{align}\label{align:ypolylin}
    \Ypolylin = \left(\sqrt{\theta_2^Q}S + \sqrt{\theta_2^Q-\theta_1^Q}W \right)/\sqrt{d},
\end{align}
where $\theta_1^Q=\theta_1(Q_n)$ and $\theta_2^Q=\theta_2(Q_n)$.  
For each fixed $n$, let $\nu_{n,\infty}=\lim_{d\to\infty}\nu_{\Ypolylin}$ denote the limiting singular-value distribution.

\begin{lemma}[Gaussian Equivalence]\label{lem:lin-poly}
For any $q\in\mathbb{N}$, we have the self-averaging
\begin{align}\label{align:as-conv}
   m_q\!\left(\nu_{\Ypoly}\right) \xrightarrow[d\to\infty]{\text{a.s.}} \E\!\left[m_q\!\left(\nu_{\Ypoly}\right)\right].    
\end{align}
Moreover, as $d \to \infty$,
\begin{align}
   \E\!\left[m_q\!\left(\nu_{\Ypoly}\right)\right]
   = \bigl(1+o(1)\bigr)m_q\!\left(\nu_{\infty}\right).
\end{align}
\end{lemma}

\begin{proof}
The almost-sure convergence~\eqref{align:as-conv} follows from~\citep[Lemma~3.10]{benigni2021eigenvalue}.  
Take $n_d=c\log d/\log\log d$ as \eqref{align:nd}, \citet[Theorem~3.5]{benigni2021eigenvalue} yields as $d \to \infty$,
\begin{align}
   \E\!\left[m_q\!\left(\nu_{\Ypoly}\right)\right]
   = \bigl(1+o(1)\bigr)\, m_q\!\left(\nu_{n_d,\infty}\right).
\end{align}
Finally, by \cref{lem:nu-n-infty}, $m_q(\nu_{n_d,\infty})\to m_q(\nu_\infty)$, which proves the assertion.
\end{proof}

\subsection{Proof of Theorem}\label{ssec:proof-of-thm}
We are now ready to prove the main theorem. Here we use the order of $\eps_d$ in \cref{lem:main-concentration}. The argument proceeds by tracing back from $\Ylin$, for which the limiting distribution $\nu_\infty$ is known to exist.

\begin{proof}[Proof of \cref{theorem:main}]
Fix $q\in\mathbb{N}$. We need to show that almost surely as $d\to\infty$,
\begin{align}\label{align:main-claim}
   m_q\!\bigl(\nu_{\sqrt{d}(A-u_du_d^\top)}\bigr) \to m_q(\nu_\infty).
\end{align}
By \cref{lem:nu-n-infty,lem:lin-poly}, we have almost surely $m_q(\nu_{\Ypoly})$ has the same limit as $ m_q(\nu_{\Ypolylin})$, i.e.\,$m_q(\nu_\infty)$.
Since $\nu_\infty$ has finite moments of all orders, \cref{lem:stong-conv,lem:stong-and-moments} further yields a.s.$ m_q(\nu_{\Ycntr})$  has the same limit as $m_q(\nu_{\Ypoly})$.

Next, by \cref{lem:main-concentration}, writing $\mathcal{E}=D^{-1}-I_d$, 
\begin{align}
   \sqrt{d}(A-u_du_d^\top) = (I_d+\mathcal{E})\Ycntr + \sqrt{d}\mathcal{E}u_du_d^\top,
\end{align}
where $\mathcal{E}$ is diagonal with $\max_i |\mathcal{E}_{ii}|\le \varepsilon_d=o(d^{-1/2+\delta})$. We bound the rank-one term:
$m_q(d\mathcal{E}^2 u_du_d^\top)^{1/q}\le d m_q(u_du_d^\top)^{1/q}\varepsilon_d^2.$
Since
$ m_q(u_du_d^\top)
   =\frac{1}{d}\|u_d\|_2^2
   =\frac{1}{d}$,
we obtain
\begin{align}
   m_q\!\bigl(d\,\mathcal{E}^2 u_du_d^\top\bigr)^{1/q}
   \le d^{\,1-1/q}\,\varepsilon_d^2.
\end{align}
Choosing $\delta=1/(2q)$ and using $\varepsilon_d=o(d^{-1/2+\delta})$ gives $d^{\,1-1/q}\varepsilon_d^2=o(1)$, hence
$m_q\bigl(d\,\mathcal{E}^2 u_du_d^\top\bigr)^{1/q}\to 0.$
Therefore, by the triangle inequality applied to $L^q$-norm, a.s., 
$m_q\!\bigl(\nu_{\sqrt{d}A^\perp}\bigr)
   \to 
   \lim_{d\to\infty} m_q\!\bigl(\nu_{(I_d+\mathcal{E})\Ycntr}\bigr)
   = m_q(\nu_\infty).$

Finally, the weak convergence of $\nu_{\sqrt{d}A}$ follows from interlacing for rank-one perturbations together with the compactness of the support of $\nu_\infty$.
\end{proof}

\subsection{Outliers and the Gap}
Let us first check $\Ylin$ has no outliers, that is, $\| \Ylin \|^2=s_1(\Ylin)^2$ converges to the right edge of $\nu_\infty$.
Set $W_1=W^Q/\sqrt{d}, W_2=(W^K)^\top /\sqrt{d}, W_3=W/\sqrt{d}$, then they are i.i.d.\,standard real Ginibre ensembles. We apply free probability theory (see \cref{ssec:free-defn} for related notations) to handle the operator norm $\| \Ylin \|$. 
%
We only need to show $\| \alpha W_1W_2 + \beta W_3\|$ converges to $\|\alpha c_1c_2 + \beta c_3\|$, where $c_1,c_2,c_3$ are $*$-free circular elements. 

To achieve this, it is enough to prove that $(W_1,W_2,W_3)$ converges strongly to $(c_1,c_2,c_3)$ in the sense of \citep{collins2014strong},
namely, for any non-commuting polynomial $P$ in three abstract non-commuting variables and its transposes,  $\|P(W_1,W_2,W_3)\|$ converges to $\|P(c_1,c_2,c_3)\|$
in the large $d$-limit. 

Note that in the case of complex Ginibre, this is already known, as the real part and imaginary part of a complex Ginibre are iid GUE and the strong convergence of GUE is established. 
In the orthogonal case, we can not rely on this trick and we are not available of any direct proof, although we believe that the most recent results of \citep{van2025strong} suffice to prove this result. 

The easiest way to overcome this problem is to revisit the proof of \cite{collins2014strong}, starting with the fact that
one knows the strong convergence of random orthogonal i.i.d Haar distributed matrices (real) $V_i$.
Then, we make the following observation:
There exists $c>0$ such that for any $\varepsilon >0$, there exists a real polynomial $P_\varepsilon$ and i.i.d orthogonal matrices $V_1,V_2,V_3$ such that
$P(\big| \| V_1P_\varepsilon (V_2+V_2^\top)V_3\|-2\big|>\varepsilon)=O(\exp (-c\varepsilon^2d))$.

The existence of such a polynomial follows closely the proof of \cite{collins2014strong}. Namely, one can find a continuous increasing function $P$ such that the pushforward under $\mu$ of the arcsine distribution (the asymptotic eigenvalue counting distribution of $V+V^\top$) is the quarter-circle distribution. 
It follows from the strong convergence of a single real Ginibre matrix that this function $P$ satisfies 
$P(\big \| V_1P (V_2+V_2^\top)V_3\| - 2 \big| >\varepsilon)=O(\exp(-c\varepsilon^2 d))$. 
(Note that $\| W_i\| \to 2$, the right edge of the quarter-circular law.)
%
The next step consists of replacing the continuous function $P$ by an increasing real polynomial $P_\varepsilon$ that is uniformly close to $P$ on an open set containing the closed interval $[-1,1]$

We repeat this argument for $W_2,W_3$ and observe that all the Haar orthogonal variables involved $V_1,\ldots , V_9$ can be chosen to be i.i.d.
Using a $3\varepsilon$ argument like in \cite{collins2014strong}, we get strong convergence and therefore the expected result.

We conclude that the potential outlier of $A$ is caused by the rank-one perturbation and the nonlinearity,  because of the fact that $\Ylin$ has no outlier asymptotically. 
Further, by \cref{prop:strict-lb}, it holds that 
\begin{align}\label{align:bound-right-edge}
\lim_{d\to\infty}\|\Ylin \|^2=\max \mathrm{supp} (\nu_\infty) > 4(e^{\beta^2}-1),
\end{align}
From this, the right edge of the bulk $\nu_\infty$ of $dAA^\top$ is strictly larger than that of the corresponding Marchenko-Pastur law ($=4\theta_1$) with the i.i.d.\,assumption in \cite{saada2025mind}, showing an effect of non-independent term $S$ in the attention matrix: it enlarges the bulk and decreases the gap between the bulk and $s_1(A)$.

\section{Numerical Simulations}

\subsection{Empirical Spectra}\label{ssec:empirical-six}

To verify, at finite $d$, that each approximation step in the proof chain is accurate,
we compared six random matrix models against \(\sqrt{d}\,A\) by examining the squared singular values.
We set \(n_T=d=d_{qk}=1000\), used \(n_s=10\) samples, and fixed \(\beta=1\).
For the polynomial approximation, we chose \(c=2\) and \(n_d=8\). \cref{sec:numerical-settings} summarizes the settings.
\cref{fig:empirical-hist-six} shows histograms and sorted curves of the squared singular values after removing the top three from each sample (the top three are shown separately in \cref{fig:topk}).
Across the six models, the bulks agree closely, consistent with the predicted limiting law.
The bulk right edge exceeds \(4\theta_1=4(e^{\beta^2}-1)\) at \(\beta=1\), providing empirical support for~\eqref{align:bound-right-edge}.
%
\begin{wrapfigure}[11]{hr}{0.49\linewidth}
    \centering
    \vspace{-\intextsep}
    \includegraphics[width=\linewidth]{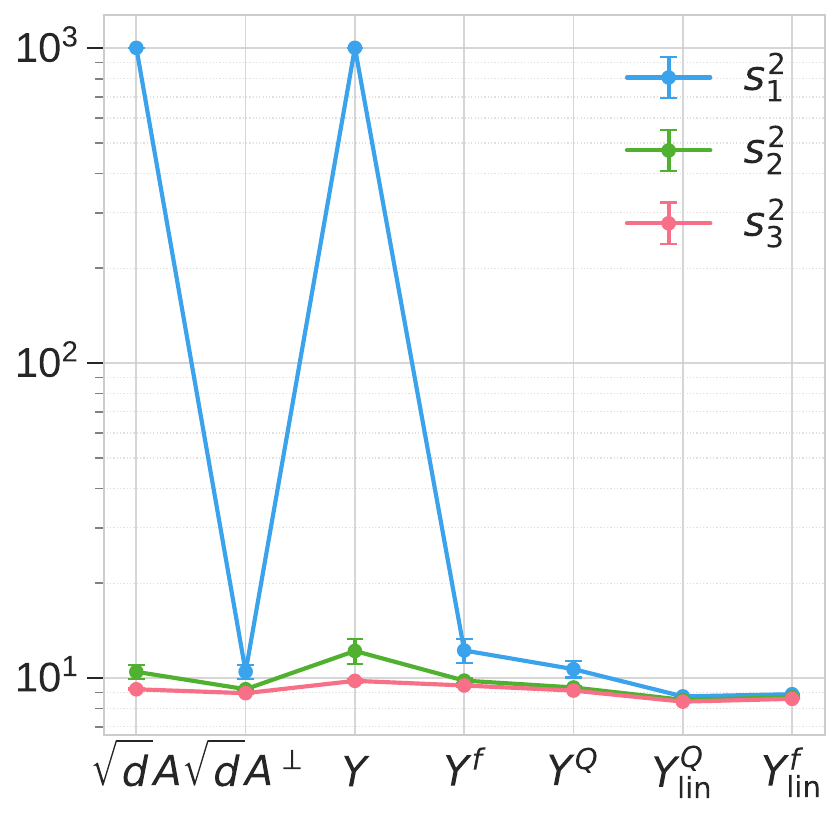}
    \captionsetup{skip=2pt}
    \caption{Plots of $s_1^2, s_2^2, s_3^2$.}
    \label{fig:topk}
\end{wrapfigure}    
\cref{fig:topk} reveals a very large leading value \(s_1^2\sim d\) for \(\sqrt{d}\,A\) and \(Y\), as predicted.
Even after removing this dominant outlier, the nonlinear models
\(\sqrt{d}\,A\), \(\sqrt{d}(A-u_du_d^\top)\), \(Y\), \(Y^{f}\), \(Y^{Q}\)
show additional outliers relative to the linearized models
\(Y^{f}_{\mathrm{lin}}\), \(Y^{Q}_{\mathrm{lin}}\).
A mechanism-level discussion is deferred to \cref{ssec:discussion-outliers}.

\subsection{Signal--Noise Balance}

\begin{figure}[t]
\centering
\includegraphics[width=0.48\linewidth]{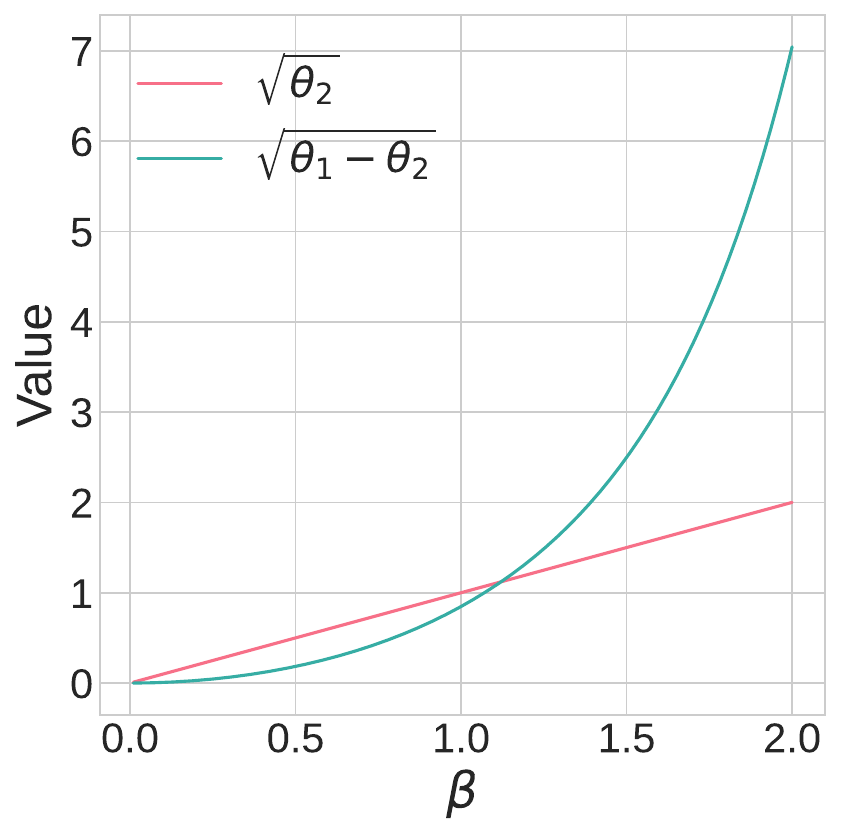}
\includegraphics[width=0.48\linewidth]{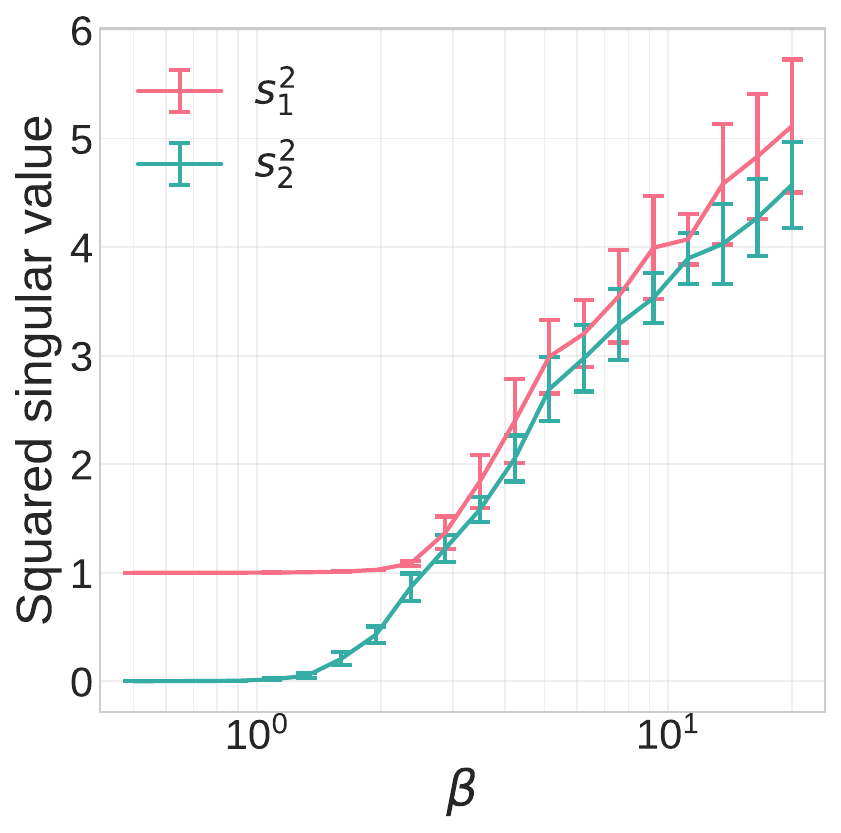}
\caption{(Left): Growth of coefficients around $\beta=1$. (Right): Growth of $s_1(A)^2$ and $s_2(A)^2$ with $d=1000$ including $\beta \gg 1$.}
\label{fig:balance}
\end{figure}    
%
We quantify how the two Gaussian‑equivalence components shape $\nu_\infty$ across $\beta$; the balance between $\sqrt{\theta_2}$ and $\sqrt{\theta_1-\theta_2}$ dictates whether the model behaves more like a single Ginibre matrix or like $S/\sqrt{d}$.
We examined how the comparison of the two scales varies with the inverse temperature \(\beta\).
In our setting, \(\sqrt{\theta_2}=\beta\) and \(\sqrt{\theta_1-\theta_2}=\sqrt{e^{\beta^2}-1-\beta^2}\).
\cref{fig:balance}~(left) shows that these curves intersect at \(\beta=1.121\pm 10^{-3}\).

\subsection{Large Inverse Temperature}
To probe $\beta$ well beyond the valid range of our bounds and demonstrate why those bounds matter, we increased \(\beta\) at fixed \(d=1000\) with \(n_s=10\).
\cref{fig:balance}~(right) shows that \(s_2(A)^2\) rises near \(\beta\gtrsim 1\), while \(s_1(A)^2\) grows for \(\beta\gtrsim 2\),
indicating that the softmax linearization breaks before the normalizer dominates.
At \(\beta=50\), \cref{fig:vs_poisson} exhibits discrete atoms in \(\nu_A\) aligned with Poisson(1) quantiles, consistent with an argmax-like limit.
Theoretical thresholds and the Poisson heuristic are discussed in \cref{ssec:discussion-beta}.

\begin{figure}[t]
    \centering
    \includegraphics[width=0.95\linewidth]{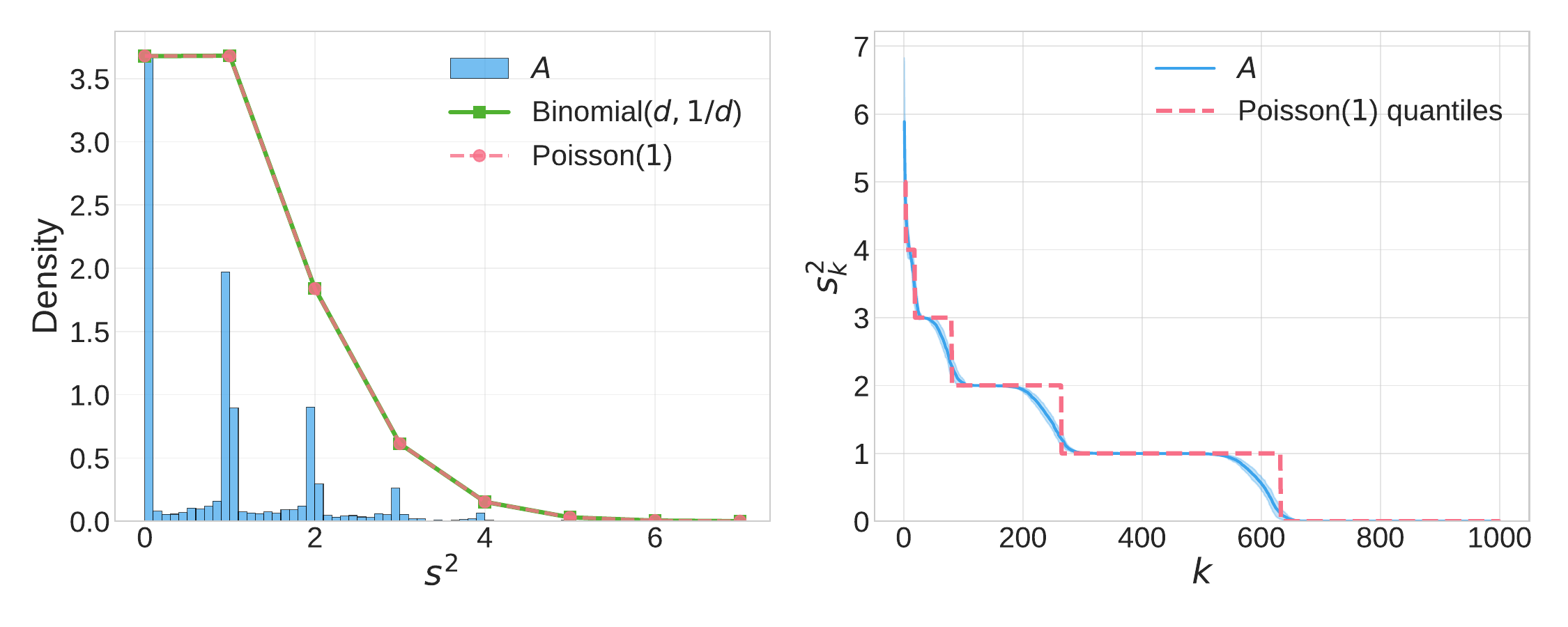}
    \caption{$\nu_A$ vs.\ Poisson$(1)$. Left: histogram ($d=1000$, $\beta=50$, $n_s=10$, bin width $0.1$); right: sorted value and quantile plot using $F^{-1}(1-k/d)$ for Poisson$(1)$. Histogram heights differ because Poisson is discrete, but the quantiles match well.}
    \label{fig:vs_poisson}
\end{figure}

\section{Discussion}
\subsection{Threshold in The Inverse Temperature}\label{ssec:discussion-beta}
Our rigorous analysis establishes rank collapse for $\beta=O(1)$.
We argue for a threshold $\beta=\Theta(\sqrt{\log \ntoken})=\Theta(\sqrt{\log d })$ that separates regimes. (Recall we take a proportional limit of $d$ and $\ntoken$.)
From the fluctuation scale \(d^{-1/2+\delta}\) of the row normalizer, the truncation level \(K= (\log d)^{1/2+\delta}\),
and the observation that the bulk right edge scales like \(O(e^{\beta^2})\)~\eqref{align:bound-right-edge}, the evidence points to a threshold at \(\beta=\Theta(\sqrt{\log d})\).
Around this scale, the softmax linearization fails and the bulk right edge
inflates, so additional singular directions can rise to $O(1)$.
This conclusion is consistent with \cref{fig:balance}(right) (growth of $s_2(A)^2$ near $\beta\gtrsim 1$).

In the further large $\beta$ regime, \cref{fig:vs_poisson} illustrates that \(s_1(A)\) departs from \(1\) and \(\nu_A\) develops Poisson-like atoms, consistent with an argmax limit at very large \(\beta\). For $\beta \gg 1$, the softmax is approximated by the \emph{argmax}. The distribution of squared singular values associated with a random argmax coincides with that of drawing $d$ balls independently with replacement into $d$ bins, which is $\mathrm{Binomial}(d,1/d)$; this converges to a Poisson distribution with rate $\lambda=1$ as $d\to\infty$. Thus, we hypothesize that for very large $\beta$, $\nu_A$ is well approximated by $\mathrm{Poisson}(1)$. 

\subsection{General Input Matrix}
We outline how to extend beyond orthonormal $X$ to arbitrary $X$.
Generalizing the input matrix $X$ should be attainable by considering a four-factor random-matrix model $XW^Q (XW^K)^\top$ combined with an entrywise nonlinearity.  
In particular, an extension of \citep{speicher2024entrywise} should accommodate this setting.

\subsection{Outliers}\label{ssec:discussion-outliers}
We explain outlier mechanisms, the bulk misses, and sketch their modeling.
Interlacing alone does not control extreme eigenvalues; therefore, so edge behavior requires separate analysis.
Nonlinearity is likely to generate outliers. Following the suggestion of~\cite{benigni2022largest} under strong assumptions on the nonlinearity,
introducing a rank-one term with weight \(\sqrt{\theta_3}\), where \(\theta_3=\mathbb{E}\big[\tfrac12 f''(\chi)\big]^2\),
may capture such outliers.
A rigorous treatment would need uniform high-moment bounds \(q=q(d)\) together with uniform control of normalizer fluctuations,  with a slowly growing function dictated by the LIL~\eqref{align:lil}.

\section{Conclusion}

This work establishes a Gaussian equivalent model for the bulk singular value distribution of \(\sqrt{d}\,A\) in the \(\beta=O(1)\) regime.
The proof relies on (i) the concentration of normalizers at scale \(d^{-1/2+\delta}\), (ii) a tight polynomial approximation of \(\exp(\beta x)\) with degree \(n_d=\lceil c\log d/\log\log d \rceil\), and (iii) a polynomial linearization transferring limits to \(f(S)\) and then to \(A\).
The resulting bulk diverges from the Mar\v{c}enko–Pastur law due to the non-independence in \(S\), and our simulations (\cref{fig:empirical-hist-six}) corroborate the enlarged right edge.
Open directions include  the large-\(\beta\) regime, extending to general \(X\), multi-head attention, and a precise theory of outliers which will require uniform high-moment bounds.

\section*{Acknowledgements}
We thank the anonymous reviewers for their constructive feedback, which helped us improve the introduction and related work sections.
T.H.\ was supported by JST BOOST, Japan, Grant Number JPMJBY24G4.
R.K.\ was supported by JST FOREST, Japan, Grant Number JPMJFR226Q, and JSPS KAKENHI Grant Numbers 22H05116 and 23K16965.

\bibliographystyle{hunsrtnatarxiv}

\section*{Checklist}



\begin{enumerate}

  \item For all models and algorithms presented, check if you include:
  \begin{enumerate}
    \item A clear description of the mathematical setting, assumptions, algorithm, and/or model. [Yes]
    \item An analysis of the properties and complexity (time, space, sample size) of any algorithm. [Not Applicable]
    \item (Optional) Anonymized source code, with specification of all dependencies, including external libraries. [Not Applicable]
  \end{enumerate}

  \item For any theoretical claim, check if you include:
  \begin{enumerate}
    \item Statements of the full set of assumptions of all theoretical results. [Yes]
    \item Complete proofs of all theoretical results. [Yes]
    \item Clear explanations of any assumptions. [Yes]     
  \end{enumerate}

  \item For all figures and tables that present empirical results, check if you include:
  \begin{enumerate}
    \item The code, data, and instructions needed to reproduce the main experimental results (either in the supplemental material or as a URL). [Yes]
    \item All the training details (e.g., data splits, hyperparameters, how they were chosen). [Not Applicable]
    \item A clear definition of the specific measure or statistics and error bars (e.g., with respect to the random seed after running experiments multiple times). [Yes]
    \item A description of the computing infrastructure used. (e.g., type of GPUs, internal cluster, or cloud provider). [Not Applicable]
  \end{enumerate}

  \item If you are using existing assets (e.g., code, data, models) or curating/releasing new assets, check if you include:
  \begin{enumerate}
    \item Citations of the creator If your work uses existing assets. [Not Applicable]
    \item The license information of the assets, if applicable. [Not Applicable]
    \item New assets either in the supplemental material or as a URL, if applicable. [Not Applicable]
    \item Information about consent from data providers/curators. [Not Applicable]
    \item Discussion of sensible content if applicable, e.g., personally identifiable information or offensive content. [Not Applicable]
  \end{enumerate}

  \item If you used crowdsourcing or conducted research with human subjects, check if you include:
  \begin{enumerate}
    \item The full text of instructions given to participants and screenshots. [Not Applicable]
    \item Descriptions of potential participant risks, with links to Institutional Review Board (IRB) approvals if applicable. [Not Applicable]
    \item The estimated hourly wage paid to participants and the total amount spent on participant compensation. [Not Applicable]
  \end{enumerate}

\end{enumerate}

\clearpage
\appendix
\thispagestyle{empty}

\onecolumn
\aistatstitle{Supplementary Materials}
\newcommand{\wt}{\widetilde}

\section{Notations}
Throughout this paper, we fix a probability space $(\Omega,\mathcal{F}, \Pr)$. We say an event $E \in \mathcal{F}$ occurs \emph{almost surely} if $\Pr(E)=1$.

\begin{definition}
    For a sequence of probability distributions $\nu_d (d \in \N)$ and $\nu$, 
    We say $\nu_d$ converges to $\nu$ \textit{weakly} if   
    \begin{align}
       \lim_{d \to \infty} \int g(x)\nu_d(dx)  = \int g(x) \nu(dx),
    \end{align}
    for any bounded continuous function $g$ on $\R$.

   We say $\nu_d$ ($d \in \N$) converges to $\nu$ \textit{in moments} if   
    \begin{align}
            \lim_{d \to \infty}  \int x^q\nu_d(dx) = \int x^q \nu(dx),
    \end{align}
    for any $q \in \N$.    
\end{definition}

For $x \in \R^d$, we write $\| x \|_2 = \sqrt{\sum_{i=1}^d x_i^2}$. If there is no confusion, we omit the index $2$.  For matrix $A \in \R^{d \times d}$, we define the operator norm as follows:
\begin{align}
    \| A\|_\infty = \sup_{x \in \R^d, \|x\|_2=1 }\| Ax\|_2.
\end{align}
In fact, $\| AA^\top \|_\infty= s_1(A)^2$, where $A^\top$ is the transposed matrix of $A$.
Since in this paper we only consider the operator norm  for matrices, we omit the index $\infty$ if there is no confusion.

\begin{definition}[Asymptotic notation]
Let $f:\mathbb{N}\to\mathbb{R}_+$ and $g:\mathbb{N}\to\mathbb{R}_+$ be two nonnegative functions, and assume $f(d)>0$ for all sufficiently large $d$.
We write 
\begin{align}
g(d)=\Theta(f(d)) \quad \text{as } d\to\infty
\end{align}
if there exist constants $c_1,c_2>0$ and $d_0\in\mathbb{N}$ such that 
$c_1 f(d)\le g(d)\le c_2 f(d), \forall d\ge d_0.$
We write 
\begin{align}
g(d)=O(f(d)) \quad \text{as } d\to\infty
\end{align}
if there exists a constant $C>0$ and $d_0\in\mathbb{N}$ such that
$0\le g(d)\le C\,f(d), \forall d\ge d_0.$
We write 
\begin{align}
g(d)=o(f(d)) \quad \text{as } d\to\infty
\end{align}
if for every $\varepsilon>0$ there exists $d_0\in\mathbb{N}$ such that
$0\le g(d)\le \varepsilon\,f(d), \forall d\ge d_0,$
equivalently, $\lim_{d\to\infty}\tfrac{g(d)}{f(d)}=0.$
\end{definition}

\section{Lemmas}

\newcommand{\nullc}{\Omega_{1,\trunc/\beta}}
\newcommand{\nullca}{\Omega_{d,\trunc/\beta}}

\newcommand{\Gd}{G_\delta} 
\newcommand{\ce}[2]{\mathcal L_{#1}(#2)} 

\newcommand{\xiij}{\xi_j}

\subsection{Concentration of Normalizer}

\subsubsection{A Single Row}
Let $v,w$ be i.i.d.\,standard Gaussian vectors.
Write $ \xi= \langle v, w \rangle /\sqrt{d}$.
Further, write 
\begin{align}
    \sigma = \| v\|/\sqrt{d}.
\end{align}
Define $\Omega_\trunc=\{ |\xi| \leq \trunc\}$.
\begin{lemma}\label{lem:chernov}
    $\Pr(\Omega_\trunc^c)\leq  2\exp(-\trunc^2/2 )$ for $\trunc <\sqrt{d}$.
\end{lemma}
\begin{proof}
    By Chernoff, $\Pr(\Omega_\trunc^c) \leq \inf_{0 < \lambda < \sqrt{d}} \exp \left( -\lambda \trunc - \frac{d}{2} \log (1 -\lambda^2/d) \right)$. Now $-d  \log(1-\lambda^2/d) \leq \lambda^2$ since $a + \log(1-a)\geq 0$ for $0<a<1$. By $\trunc < \sqrt{d}$, we can set $\lambda=\trunc$ and get
    $\Pr(\Omega_K^c)\leq 2 \exp(-\trunc^2 + \trunc^2/2)$.
\end{proof}

\begin{lemma}
\begin{align}
\E[e^\xi \mid \Omega_\trunc, v] =\ce{\trunc}{\sigma}:=e^{\sigma^{2}/2}\,
\frac{\Phi\left(\frac{\trunc-\sigma^{2}}{\sigma}\right)-\Phi\left(\frac{-\trunc-\sigma^{2}}{\sigma}\right)}
{\Phi\left(\frac{\trunc}{\sigma}\right)-\Phi\left(\frac{-\trunc}{\sigma}\right)}    
\end{align}
where $\Phi$ is the standard normal CDF.
\end{lemma}

\begin{proof}
Recall $\sigma^2=\| v \|^2/d$. Then
    \begin{align}
     \E[e^\xi \mid \Omega_\trunc, v] = \frac{\displaystyle \int_{-\trunc}^{\trunc} e^{x}\,\frac{1}{\sigma\sqrt{2\pi}}e^{-x^{2}/(2\sigma^{2})}\,dx}
{\displaystyle \int_{-\trunc}^{\trunc}\frac{1}{\sigma\sqrt{2\pi}}e^{-x^{2}/(2\sigma^{2})}\,dx}.
    \end{align}
Thus, we complete the square in the numerator:
\begin{align}
e^{x-\frac{x^{2}}{2\sigma^{2}}}
=e^{\sigma^{2}/2}\, \exp\!\left(-\frac{(x-\sigma^{2})^{2}}{2\sigma^{2}}\right).    
\end{align}
Thus, the assertion holds.
\end{proof}

Therefore, the conditional expectation only depends on $\trunc$ and $\sigma=\| v \|/\sqrt{d}$.
As $\trunc\to\infty$, the ratio $\to 1$, so it converges to $\mathbb E[e^\xi\mid v]=e^{\sigma^{2}/2}$.

\begin{lemma}
For $\beta>0$, we have $\E[e^{\beta \xi} \mid \Omega_{K/\beta}, v] =  \ce{\trunc}{\beta\sigma}$ with $\sigma=\| v\|_2/\sqrt{d}$.
\end{lemma}
\begin{proof}
By replacing $v$ with $\beta v$, the assertion holds.
\end{proof}

Recall $\mathbb E[e^{\beta\xi} ]=e^{\beta^2/2}$. 

\begin{lemma}\label{lem:prob-con-general}
Fix  $\delta \in (0,1)$. Then
$$
\Pr\left(\big| \ce{\trunc}{\beta\sigma}-e^{\beta^2/2}\big|>\varepsilon\right)
\le 2\exp\!\left(
-\frac{d}{2}\,
\left(\frac{(\varepsilon-\Delta^\beta(\trunc))_+}{L^\beta_\delta(\trunc)}\right)^2
\right) + 2e^{-d\delta^2/2}.
$$
where $(x)_+=\max(x,0)$,
\begin{align}
 \Delta^\beta(\trunc)&=|\ce{K}{\beta}- e^{\beta^2/2}|,   \\
L^\beta_\delta(\trunc)&= \beta \sup_{|\sigma -1|\leq \delta}|\mathcal{L}_{\trunc}^\prime(\beta\sigma)|.
\end{align}
\end{lemma}

\begin{proof}
Consider $G_\delta=\{ |\sigma-1|<\delta\}$.
On $G_\delta$, by the mean-value theorem around $\sigma=1$,
\begin{align}
 |\ce{\trunc}{\beta \sigma}-e^{\beta^2/2}| \leq \Delta^\beta(\trunc) +   L^\beta_\delta(\trunc) |\sigma-1|.
\end{align}
Hence, for any $\varepsilon>0$,
\begin{align}
    \Pr \left( \big| \ce{\trunc}{\beta \sigma} - e^{\beta^2/2}\big| > \eps\right)
    \le \Pr\!\left(|\sigma-1|>\frac{(\varepsilon-\Delta^\beta(\trunc))_+}{ L^\beta_\delta(\trunc)}\right) + \Pr(G_\delta^c).
\end{align}

Now use the $\chi$-concentration for $\sigma=\|v\|/\sqrt d$:
\begin{align}
\Pr(|\sigma-1|>u)\le2\exp\!\left(-\frac{d\,u^2}{2}\right)\qquad(0<u\le 1/2).     
\end{align}
Therefore, we have the claim.
\end{proof}

\begin{lemma}[Bias bound for truncated normalizer]\label{lem:bias-bound}
Assume $K/\beta \ge 3/2$. Let 
\[ \Delta^\beta(K) := \big|\mathcal L_K(\beta) - e^{\beta^2/2}\big| \] 
be the bias of the truncated normalizer at $\sigma=1$. Then 
\[ 
\Delta^\beta(K) \le C_{\mathrm{bias}}(\beta)\,\exp\!\Big(-\frac{(K-\beta^2)^2}{2\beta^2}\Big)\!, 
\] 
where $C_{\mathrm{bias}}(\beta) = \frac{\beta\,e^{\beta^2}}{\sqrt{2\pi}\,\big(2\Phi(3/2)-1\big)} = O(e^{\beta^2})$. In particular, if 
\[ 
K \ge \beta^2 + \beta\,\sqrt{\,2\log\!\frac{2\,C_{\mathrm{bias}}(\beta)}{\varepsilon}\,}\,, 
\] 
then $\Delta^\beta(K) \le \varepsilon/2$.
\end{lemma}
\begin{proof}
Write $\rho:=K/\beta$. Recalling that $\Phi$ is the standard normal CDF, we can express the bias as 
\[
\frac{\Delta^\beta(K)}{e^{\beta^2/2}} = \frac{2\Phi(\rho)\,-\,\Phi(\rho-\beta)\,-\,\Phi(\rho+\beta)}{\,2\Phi(\rho)-1\,}\,. 
\] 
Using the inequality $e^{-t^2/2}\sinh(\rho t)\le (\rho t)\,e^{-t^2/2}$ for $0\le t\le \beta$, one finds that 
\[ 
2\Phi(\rho)\,-\,\Phi(\rho-\beta)\,-\,\Phi(\rho+\beta) \le \frac{1}{\sqrt{2\pi}}\,\frac{\beta}{K}\,\exp\!\Big(-\frac{(K-\beta^2)^2}{2\beta^2} + \beta^2\Big)\!. 
\] 
Moreover, $2\Phi(\rho)-1 \ge 2\Phi(3/2)-1$ under the assumption $\rho \ge 3/2$. Combining these facts, we obtain 
\[ 
\Delta^\beta(K) \le \frac{\beta\,e^{\beta^2}}{\sqrt{2\pi}\,\big(2\Phi(3/2)-1\big)\,K}\,\exp\!\Big(-\frac{(K-\beta^2)^2}{2\beta^2}\Big) \le C_{\mathrm{bias}}(\beta)\,\exp\!\Big(-\frac{(K-\beta^2)^2}{2\beta^2}\Big)\!, 
\] 
as claimed. In particular, solving $\Delta^\beta(K)\le\varepsilon/2$ for $K$ yields the stated threshold on $K$, above which $\Delta^\beta(K)\le \varepsilon/2$.
\end{proof}

\begin{lemma}[Local Lipschitz continuity of $\mathcal L_K$]\label{lem:lipschitz}
Assume $K/\beta \ge 3/2$. Then for all $\sigma$ with $|\sigma-1|\le 1/2$,  it holds that
\begin{align}    
\beta \big|\mathcal L_K^\prime(\beta\sigma) \big| \le K C_L(\beta), 
\end{align}
where  
\begin{align}\label{align:c-l-beta}
C_L(\beta) = e^{9\beta^2/8}(c_0/\beta + c_1 + c_2\beta),
\end{align}
and $c_1, c_2, c_3>0$ are constants.
In particular,  $L_{1/2}^\beta(K) \leq KC_L(\beta)$.
\end{lemma}
\begin{proof} We focus on $\mathcal{L}^\prime_K(u)$ with $u\in[\beta/2,\,3\beta/2]$, which corresponds to $|\sigma-1|\le 1/2$. In this range one checks that $D(u):=\Phi(K/u)-\Phi(-K/u) \ge 2\Phi(1)-1$ (since $K/u \ge K/(3\beta/2)\ge 1$ by assumption) and $|N(u)| \le 1$, where $N(u) := \Phi\!\big(\frac{K-u^2}{u}\big) - \Phi\!\big(\frac{-K-u^2}{u}\big)$. Writing $\mathcal L_K(u) = e^{u^2/2}\frac{N(u)}{D(u)}$ and differentiating, one can bound the derivative as 
\[ 
\big|\mathcal L'_K(u)\big| \le e^{u^2/2}\left[\frac{|u|}{2\Phi(1)-1} + \frac{2}{\sqrt{2\pi}\,\big(2\Phi(1)-1\big)}\Big(1+\frac{K}{u^2}\Big) + \frac{2}{\sqrt{2\pi}\,\big(2\Phi(1)-1\big)^2}\,\frac{K}{u^2}\right]\!. 
\] 
For $u \in [\beta/2,3\beta/2]$, we have $e^{u^2/2} \le e^{9\beta^2/8}$, $|u| \le 3\beta/2$, and $K/u^2 \le 4K/\beta^2$. Substituting these bounds above with $\beta/K$ and simplifying, we find $\beta \mathcal L^\prime_K(u)$ is bounded by $KC_L(\beta)$ with \eqref{align:c-l-beta},
where $c_0,c_1,c_2>0$ are constants, which proves the assertion.
\end{proof}

\begin{lemma}\label{lem:ce-around-mean}
Fix $\beta>0$ and $\varepsilon\in(0,1)$. Assume the \emph{minimal natural truncation ratio} $K \ge \tfrac{3}{2}\,\beta$. There exists a threshold $K_0(\beta,\varepsilon) = \Theta\!\Big(\beta^2(1+  \sqrt{\log \frac{1}{\varepsilon}} )\Big)$ such that if $K \ge K_0(\beta,\varepsilon)$, then for all $d\in\mathbb N$,
\[ 
\Pr\!\Big(\big|\mathcal L_K(\beta\sigma) - e^{\beta^2/2}\big| > \varepsilon\Big) \le 2\exp\!\Big(-\,c(\beta)\,\frac{\varepsilon^2\,d}{K^2}\Big) + 2e^{-d/8}\!, 
\] 
where $c(\beta)>0$ is a constant that decays exponentially in $\beta^2$ (specifically, $c(\beta) = \Theta(e^{-\frac{9}{4}\beta^2})$). In particular, the deviation probability above is exponentially small in $d$ for fixed $\beta$, and $c(\beta)$ captures the $\beta$-dependence of the concentration rate.
\end{lemma}
\begin{proof}
Apply \cref{lem:prob-con-general} with $\delta=\tfrac12$. This gives the general tail bound 
\[ 
\Pr\!\Big(\big|\mathcal L_K(\beta\sigma) - e^{\beta^2/2}\big| > \varepsilon\Big) \le 2\exp\!\Bigg(-\frac{d}{2}\Bigg(\frac{\big(\varepsilon - \Delta^\beta(K)\big)_+}{\,L^\beta_{1/2}(K)}\Bigg)^{\!2}\Bigg) + 2e^{-d/8}\!, 
\] 
where $\Delta^\beta(K)$ and $L^\beta_{1/2}(K)$ are as defined in that lemma. Now, by \cref{lem:bias-bound} we choose $K \ge K_0(\beta,\varepsilon)$ so that $\Delta^\beta(K) \le \varepsilon/2$. Hence $(\varepsilon - \Delta^\beta(K))_+ \ge \varepsilon/2$. Moreover, by \cref{lem:lipschitz} we have $L^\beta_{1/2}(K) \le C_L(\beta)K$. Substituting these bounds into the expression above, we obtain 
\[ 
\exp\!\Big(-\,\frac{d}{2}\,\Big(\frac{\varepsilon/2}{\,L^\beta_{1/2}(K)}\Big)^{2}\Big) \leq  \exp\!\Big(-\,\frac{\varepsilon^2\,d}{8\,C_L(\beta)^2K^2}\Big) \,. 
\] 
Thus, by writing $c(\beta)=1/(8C_L(\beta)^2)$, we have the assertion.
\end{proof}

\subsubsection{Evaluation of Normalizer}
\begin{lemma}\label{lem:Hoeffding}
Let $v, w_1, \dots, w_d$ be i.i.d.\,standard normal $d$-dimensional vectors.
Write $\nullca= \cap_{j=1}^d\{|\langle v,w_j \rangle/\sqrt{d} | \leq  \trunc \}$ for $\trunc >0$.  Write $\sigma=\|v\|_2/\sqrt{d}$. Then  for any $\eps>0$,
\begin{align}
\Pr\left(\left|  \frac{1}{d}\sum_{j=1}^d \exp\left(\frac{\beta \langle v, w_j \rangle}{\sqrt{d}}\right)  - \ce{\trunc}{\beta\sigma} \right| > \eps \mid \nullca, v \right) \leq \exp\left( - \frac{2\eps^2}{(e^{\trunc} - e^{-\trunc})^2}d \right).
\end{align}
In particular, the bound is uniform on $v$.
\end{lemma}

\begin{proof}
    Let $\xi_j=  \exp(\beta \langle v, w_j \rangle/\sqrt{d}) \mid \nullca, v$.
    Then under the conditioning, $\xi_1,\dots, \xi_d$ is IID family and bounded: $e^{-\trunc} \leq \xi_j \leq e^{\trunc}$. Thus, by Hoeffding's inequality~\citep[Theorem~2]{hoeffding1963probability}, we have the claim.
\end{proof}

\begin{prop}\label{prop:tail-of-dominant-beta}
Let $v,w_1,\dots,w_d$ be i.i.d.\ $\mathcal N(0,I_d)$ in $\R^d$ and fix $\beta>0$.
For any $n\in\N$, $\delta\in(0,1/2]$, and $\varepsilon>0$, set
$\varepsilon_d:=d^{-1/2+\delta}\varepsilon$.
Then
\[
\Pr\!\left(\left|\frac{1}{d}\sum_{j=1}^d \exp\!\Big(\frac{\beta\langle v,w_j\rangle}{\sqrt d}\Big)
- e^{\beta^2/2}\right|>\varepsilon_d\right)=o(d^{-n})\qquad(d\to\infty).
\]
\end{prop}

\begin{proof}[Proof sketch]
Write $Z=\sum_j e^{\beta\langle v,w_j\rangle/\sqrt d}$ and $E=\{|Z/d-e^{\beta^2/2}|>\varepsilon_d\}$.
Set $\trunc:=c\,\frac{\log d}{\log\log d}$ with any constant $c>0$.
Split $\Pr(E)\le \Pr(E\cap \nullca)+\Pr(\nullca^c)$.
By \cref{lem:chernov} and $\trunc/\beta < \sqrt d$,
$\Pr(\nullca^c)\le 2d\exp(-\Theta(\trunc^2/\beta^2))$.
Conditioned on $\nullca$ and $v$, \cref{lem:Hoeffding} yields
$\exp\!\big(-\Theta(d^{2\delta-\eta(d)})\big)$ (some $\eta(d)\to0$).
For the bias, \cref{lem:ce-around-mean} implies
$\Pr(|\ce{K}{\beta\sigma}-e^{\beta^2/2}|>\varepsilon_d/2)
\le 2\exp\!\big(-\Theta(\varepsilon^2 d^{2\delta}/K^2)\big)+2e^{-d/8}$.
Each term is $o(d^{-n})$ as $d\to\infty$ since $\trunc=(\log d)/(\log\log d)$.
\end{proof}

\begin{lemma}\label{lem:as-of-dominant}
Fix $\delta \in (0,1/2]$. Then
$\max_{i=1,2,\dots, d}d^{1/2-\delta}|\frac{1}{d} \sum_{j=1}^d \exp(\beta S_{ij}) - \E[e^{\beta\chi}]| \to 0$, almost surely.
\end{lemma}

\begin{proof}
Let $\eps>0$ and  $\eps_d=d^{-1/2+\delta}\eps$.
Write $T_{i}= \frac{1}{d}\sum_{j} \exp(\beta S_{ij})- \E[e^{\beta\chi}]$. 
Since $\Pr(d^{1/2-\delta}\max_{i=1}^d |T_i| > \eps)=\Pr(\max_{i=1}^d |T_i| > \eps_d)$, 
we only need to show $\sum_d \Pr(\max_{i=1}^d |T_i| > \eps_d) < \infty$ for the almost-sure convergence.
Now $T_1, \dots, T_d$ is identically distributed; thus, $\Pr(\max_{i=1}^d |T_i| > \eps_d)\leq d \Pr( |T_1| > \eps_d)$. By \cref{prop:tail-of-dominant-beta}, 
we have $\Pr( |T_1| > \eps_d)=o(d^{-n})$ for any $n\in\N$. Choose $n>2$. Then
\begin{align}
   \sum_d \Pr(\max_{i=1}^d |T_i| > \eps_d) \leq \sum_d C  d /d^{n} < \infty,
\end{align}
with a constant $C>0$.
It has proven the claim.
\end{proof}

\subsection{Polynomial Approximation}

Fix $\beta>0$. For $n\in\mathbb{N}$ define the $n$-term Taylor polynomial and the remainder by
\begin{align}
P_n(\beta x)&:=\sum_{k=0}^{n-1}\frac{(\beta x)^k}{k!}, &
R_n(\beta x)&:=e^{\beta x}-P_n(\beta x).
\end{align}
Throughout, $\chi\sim\mathcal{N}(0,1)$ denotes a standard normal random variable, and $\mathrm{Poi}(\lambda)$ a Poisson variable with mean $\lambda>0$.

\begin{lemma}\label{lem:taylor-beta}
Let $K>0$ and $n\in\mathbb{N}$ with $n>\beta K$. Then
\begin{align}
\max_{|x|\le K}\,|R_n(\beta x)|
\le\Big(\frac{e\,\beta K}{n}\Big)^n.
\end{align}
\end{lemma}

\begin{proof}
Put $\mu:=\beta|x|$. Then
\begin{align}
R_n(\beta|x|)=\sum_{k=n}^{\infty}\frac{\mu^k}{k!}
= e^{\mu}\,\Pr\!\big(N_{\mu}\ge n\big), \qquad N_{\mu}\sim\mathrm{Poi}(\mu).
\end{align}
By the Chernoff bound for Poisson tails, for $n>\mu$,
\begin{align}
\Pr\!\big(N_{\mu}\ge n\big)\le \exp\!\big\{-n\log(n/\mu)+n-\mu\big\}.
\end{align}
Hence
\begin{align}
|R_n(\beta x)|\le R_n(\beta|x|)
\le \Big(\frac{e\,\mu}{n}\Big)^n
\le \Big(\frac{e\,\beta K}{n}\Big)^n,
\end{align}
which completes the proof.
\end{proof}

\begin{lemma}\label{lem:taylor-exp-beta}
Let $m:=\lceil n/2\rceil$ and $\lambda:=\beta^2/2$. Then
\begin{align}
\big|\mathbb{E}[e^{\beta\chi}]-\mathbb{E}[P_n(\beta\chi)]\big|
&=\sum_{r=m}^{\infty}\frac{\lambda^{\,r}}{r!}
= e^{\lambda}\,\Pr\!\big(\mathrm{Poi}(\lambda)\ge m\big).
\end{align}
In particular, if $n>\beta^2$ (equivalently $m>\lambda$), then
\begin{align}
\big|\mathbb{E}[e^{\beta\chi}]-\mathbb{E}[P_n(\beta\chi)]\big|
&\le \Big(\frac{e\,\lambda}{m}\Big)^{\!m}
= \Big(\frac{e\,\beta^2}{2m}\Big)^{\!m}
\le \Big(\frac{e\,\beta^2}{n}\Big)^{\!n/2}.
\end{align}
\end{lemma}

\begin{proof}
Using $\mathbb{E}[\chi^{2r}]=(2r)!/(2^r r!)$ and $\mathbb{E}[\chi^{2r+1}]=0$,
\begin{align}
\mathbb{E}[e^{\beta\chi}]-\mathbb{E}[P_n(\beta\chi)]
&=\sum_{k=n}^{\infty}\frac{\beta^k}{k!}\,\mathbb{E}[\chi^k]
=\sum_{r=\lceil n/2\rceil}^{\infty}\frac{\beta^{2r}}{(2r)!}\,\mathbb{E}[\chi^{2r}]
=\sum_{r=m}^{\infty}\frac{(\beta^2/2)^r}{r!}.
\end{align}
The series on the right equals $e^{\lambda}\Pr(\mathrm{Poi}(\lambda)\ge m)$. For $m>\lambda$, the Poisson Chernoff bound
\begin{align}
\Pr\!\big(\mathrm{Poi}(\lambda)\ge m\big)\le
\exp\!\big\{-m\log(m/\lambda)+m-\lambda\big\}
\end{align}
yields the stated inequality after multiplication by $e^{\lambda}$.
\end{proof}

\begin{lemma}\label{lem:exp-approx-bound-beta}
Define
\begin{align}
f(x)&:=e^{\beta x}/\mathbb{E}[e^{\beta\chi}]-1, \\
Q_n(\beta x)&:= (P_n(\beta x)-\mathbb{E}[P_n(\beta\chi)])/\mathbb{E}[e^{\beta\chi}].
\end{align}
For any $K>0$ and $n\in\mathbb{N}$ with $n>\max\{\beta K,\beta^2\}$,
\begin{align}
\max_{x\in[-K,K]}\,\big|f(x)-Q_n(\beta x)\big|
&\le e^{-\beta^2/2}\left[\Big(\frac{e\,\beta K}{n}\Big)^{\!n}
+\Big(\frac{e\,\beta^2}{n}\Big)^{\!n/2}\right].
\end{align}
\end{lemma}

\begin{proof}
By the triangle inequality,
\begin{align}
e^{\beta^2/2}\max_{|x|\le K}\,\big|f(x)-Q_n(\beta x)\big|
&\le \max_{|x|\le K}|R_n(\beta x)|
+\big|\mathbb{E}[e^{\beta\chi}]-\mathbb{E}[P_n(\beta\chi)]\big|.
\end{align}
Apply Lemma~\ref{lem:taylor-beta} to the first term and Lemma~\ref{lem:taylor-exp-beta} (with $n>\beta^2$) to the second to obtain the claim.
\end{proof}

\begin{lemma}\label{lem:nu-n-infty}
    $\lim_{n \to \infty} m_q(\nu_{n,\infty}) =m_q(\nu_{\infty})$
\end{lemma}
\begin{proof}
    We only need to show $\theta_i(Q_n) \to \theta_i(f)$ ($i=1,2$) as $d\to \infty$.
    \begin{align}
      \theta_1(f)-\theta_1(Q_n)=\E [f(\chi)^2-Q_n(\chi)^2] 
    \end{align}
     Let $R_n(x)= e^{-\beta^2/2}[e^x - P_n(x)]$.    
    Since $R_n^\prime=R_{n-1}$, by CS, 
    \begin{align}
        |\theta_1(f)-\theta_1(Q_n)|^2\leq \E [R_n(\chi)^2] \E [(2f(\chi)-R_{n-1}(\chi))^2].
    \end{align}
    \begin{align}
       \theta_2(f)-\theta_2(Q_n)= \E[f^\prime(\chi)]^2-\E[Q_n^\prime(\chi)]^2  =\E[R_{n-1}(\chi)](\E[2e^\chi - R_{n-1}(\chi)])  
    \end{align}
    Take arbitrary $K>1$. For $q=1,2$ and any $n> \max(\beta K , \beta^2)$, we have
    \begin{align}
       | \E[R_n(\chi)^q] |&\leq \E[ |R_n(\chi)|^q \mid |\chi|<K ] + \Pr(|\chi|>K) \\
       &\leq [e^{-\beta^2/2}(\Big(\frac{e\,\beta K}{n}\Big)^{\!n}
+\Big(\frac{e\,\beta^2}{n}\Big)^{\!n/2} )]^q + 2\exp(-K^2/2).
    \end{align}
Since $K$ is arbitrary,  $\lim_n \E[R_n(\chi)]=\lim_n \E[R_n(\chi)^2]=0$.
\end{proof}

\section{A Lower Bound for the Operator Norm}

We compute the lower bound of $\Ylin = \alpha Z/\sqrt{d} + \beta W^Q (W^K)^\top/d$ with $\alpha=\sqrt{\theta_1-\theta_2}$ and $\beta=\sqrt{\theta_2}$. From the strong convergence, we only need to find the maximum of the support of the limit distribution $\nu_\infty$. 

We use tools of free probability. Background on circular and $R$-diagonal elements, transforms, and free convolutions can be found in \citet{NicaSpeicher2006}. The edge characterization via $K'(w)=0$ is treated in \citet{AndersonGuionnetZeitouni2010}. The $R$-diagonal singular-value calculus underlying $d=c_2c_3$ goes back to \citet{HaagerupLarsen2000}; the Fuss--Catalan relation for $|d|^2$ is documented in \citet{BenaychGeorges2009}.

Let $c_1,c_2,c_3$ be $*$-free standard circular elements in a tracial C$^*$-probability space and let $\alpha,\beta>0$. Using hermitization and $R$-transform calculus, we show the strict lower bound.
\[
\lim_{d \to \infty} \| \Ylin\|_\infty=\bigl\|\alpha c_1+\beta c_2c_3\bigr\|>2\sqrt{\alpha^2+\beta^2}.
\]
We also record the exact stationary equation determining the norm and derive asymptotic expansions in the regimes $\beta/\alpha\to0$ and $\beta/\alpha\to\infty$.

\subsection{Definitions and Background}\label{ssec:free-defn}

We work in a tracial C$^*$-probability space $(\A,\|\cdot \|, \tr)$. Subalgebras are said to be free (resp.\,$*$-free) if all alternating products (resp.\,alternating products with choice of adjoint) of centered elements have zero trace \citep{Voiculescu1991,NicaSpeicher2006}. A \emph{standard semicircular} element has density $(2\pi)^{-1}\sqrt{4-x^2}$ on $[-2,2]$. A \emph{standard circular} element is $c=(s_1+is_2)/\sqrt2$ with $s_1,s_2$ free standard semicirculars \citep{NicaSpeicher2006}; its operator norm equals $2$.

For any $a\in\A$, the \emph{hermitization} is the self-adjoint element
\[
\wt a=\begin{pmatrix}0 & a\\ a^* & 0\end{pmatrix}\in M_2(\A),
\qquad \|a\|=\|\wt a\|.
\]
For self-adjoint $x$, write the Cauchy transform $G_x(z)=\tr\bigl[(z-x)^{-1}\bigr]$, the $K$-transform $K_x=G_x^{\langle-1\rangle}$ (functional inverse in a neighborhood of $0$), and the $R$-transform $R_x(w)=K_x(w)-1/w$. If $x,y$ are free and self-adjoint then $R_{x+y}=R_x+R_y$ \citep{NicaSpeicher2006}. For compactly supported laws, the right endpoint of the spectrum of $x$ is $K_x(w_*)$, where $w_*>0$ is the unique real solution of $K_x'(w_*)=0$; this is the standard ``edge equation'' characterization \citep[see, e.g.,][Ch.~2]{AndersonGuionnetZeitouni2010}.

Two basic inputs will be used repeatedly (standard facts; see references cited inline):
\begin{enumerate}
\item If $c$ is standard circular, then $\wt c$ is semicircular. Hence $R_{\alpha\wt c}(w)=\alpha^2w$ and $\|\alpha c\|=2\alpha$ \citep{NicaSpeicher2006}.
\item If $d=c_2c_3$ is the product of two free standard circulars, then $d$ is $R$-diagonal and $|d|^2$ has the (order~2) Fuss--Catalan law, i.e. the free multiplicative square of the Marchenko--Pastur law. Its moment series $M(z)$ satisfies $M(z)=1+zM(z)^3$ \citep{BenaychGeorges2009}. Using the block resolvent identity $G_{\wt d}(z)=z\,G_{|d|^2}(z^2)$ (spectral symmetrization of singular values), one obtains
\begin{equation}\label{eq:R-tilde-d}
R_{\beta\wt d}(w)=\frac{\beta^2 w}{1-\beta^2 w^2},\qquad 0<w<1/\beta,
\end{equation}
see also \citet{HaagerupLarsen2000} for $R$-diagonal calculus and singular-value relations.
\end{enumerate}

\subsection{Hermitization}

Let $X=\alpha c_1+\beta c_2c_3$ with $c_1,c_2,c_3$ $*$-free standard circulars and $\alpha,\beta>0$. By freeness,
\[
\wt X=\alpha\,\wt{c_1}\ \boxplus\ \beta\,\wt{(c_2c_3)}.
\]
Using the two inputs above, its $K$-transform is
\begin{equation}\label{eq:K-total}
K(w)=\frac1w+\alpha^2 w+\frac{\beta^2 w}{1-\beta^2 w^2},\qquad 0<w<1/\beta.
\end{equation}
Let $w_*\in(0,1/\beta)$ be the unique solution to $K'(w_*)=0$. Then
\[
\|X\|=\|\wt X\|=K(w_*).
\]
It is convenient to also record the dimensionless stationary equation: writing $y=\beta^2 w^2$ and $\gamma=(\alpha/\beta)^2$, the equation $K'(w)=0$ is equivalent to
\begin{equation}\label{eq:cubic}
\gamma y^3-2\gamma y^2+(\gamma+3)y-1=0,\qquad y\in(0,1),
\end{equation}
which has a unique solution in $(0,1)$.

\subsection{A Strict Lower Bound}

\begin{prop}\label{prop:strict-lb}
For all $\alpha,\beta>0$,
\[
\bigl\|\alpha c_1+\beta c_2c_3\bigr\|>2\sqrt{\alpha^2+\beta^2}.
\]
\end{prop}

\begin{proof}
By \eqref{eq:K-total} and $K'(w_*)=0$,
\[
\|X\|=K(w_*)=\frac1{w_*}+\alpha^2 w_*+\frac{\beta^2 w_*}{1-\beta^2 w_*^2}
=\frac1{w_*}+(\alpha^2+\beta^2)w_*+\frac{\beta^4 w_*^3}{1-\beta^2 w_*^2}.
\]
Since $w_*\in(0,1/\beta)$, the last term is strictly positive. Therefore
\[
\|X\|>\frac1{w_*}+(\alpha^2+\beta^2)w_*.
\]
For any $A>0$ and $w_*>0$, the AM--GM inequality gives $\frac1{w_*}+A w_*\ge 2\sqrt{A}$. Taking $A=\alpha^2+\beta^2$ yields the strict inequality
\[
\|X\|>2\sqrt{\alpha^2+\beta^2}.
\]
\end{proof}

\begin{remark}
The function $K$ in \eqref{eq:K-total} is strictly convex on $(0,1/\beta)$, since $K''(w)=2w^{-3}+2\beta^4 w(1-\beta^2 w^2)^{-2}+4\beta^6 w^3(1-\beta^2 w^2)^{-3}>0$. Hence $w_*$ is the unique global minimizer of $K$, and $K(w)\to\infty$ at both endpoints, so indeed $\|X\|=K(w_*)$ as used above \citep[cf.][]{AndersonGuionnetZeitouni2010}.
\end{remark}

\section{Technical Lemmas}

\begin{lemma}\label{lem:stong-and-moments}
Let $(X_d)_{d \ge 1}$ and $(Y_d)_{d \ge 1}$ be sequences of real $d \times d$ matrices. 
   Let $\nu$ be a compactly supported probability distribution on $[0,+\infty)$. Assume that $\|X_d-Y_d\|_\infty\to 0$, $m_q(YY^\top)\to m_q(\nu) (q \in \N)$ then $m_q(XX^\top) \to m_q(\nu) (q\in\N)$.
\end{lemma}
\begin{proof}
    It holds that $\max_i |s_i(X)-s_i(Y)| \leq \| X-Y\|_\infty$ and thus $C_q:=\sup_d m_q(XX^\top) < \infty$. Thus, by the Markov inequality,  $\frac{1}{d}\sum_{i:s_i\geq R} |s_i(XX^\top)^q |\leq  C_{2q}/R^2$.
   Let $\trunc_R(t)= \min(t^{2q}, R^{2q})$. By the Lipschitz bound of $\trunc_R$, the moments converge to the desired value on the good set. Since $R>0$ is arbitrary, we have proven the claim.
\end{proof}

\begin{lemma}\label{lem:strong-and-moments-v2}
Let $(X_d)$ and $(Y_d)$ be $d\times d$ real matrices and write $s_i(A)$ for the singular values of $A$.
Assume
\begin{align}
\|X_d-Y_d\|_\infty \to 0,
\qquad
m_q(Y_dY_d^\top):=\frac{1}{d}\sum_{i=1}^d s_i(Y_d)^{2q}\to m_q(\nu)\in[0,\infty)
\end{align}
for every $q\in\mathbb N$. Then for every $q\in\mathbb N$,
\begin{align}
m_q(X_dX_d^\top)\to m_q(\nu).
\end{align}
\end{lemma}

\begin{proof}
\textbf{Step 1 (uniform moment bound for $X_d$).}
By Weyl's inequality for singular values,
\begin{align}
\max_{1\le i\le d}\bigl|s_i(X_d)-s_i(Y_d)\bigr|\le \|X_d-Y_d\|_\infty.
\end{align}
Hence, for any $q\in\mathbb N$ and all $i$,
\begin{align}
s_i(X_d)^{2q}\le 2^{2q-1}\bigl(s_i(Y_d)^{2q}+\|X_d-Y_d\|_\infty^{2q}\bigr).
\end{align}
Averaging,
\begin{align}
m_q(X_dX_d^\top)\le 2^{2q-1}\Bigl(m_q(Y_dY_d^\top)+\|X_d-Y_d\|_\infty^{2q}\Bigr).
\end{align}
Since $m_q(Y_dY_d^\top)$ converges and $\|X_d-Y_d\|_\infty\to0$, we obtain
\begin{align}
C_q:=\sup_d m_q(X_dX_d^\top)<\infty.
\end{align}

\textbf{Step 2 (truncation and Lipschitz control on the ``good'' set).}
Fix $q\in\mathbb N$ and $R>0$. Let $\trunc_R(t)=\min\{t^{2q},R^{2q}\}$ on $[0,\infty)$.
Then $\trunc_R$ is Lipschitz with constant $L_R=2q\,R^{2q-1}$.
By the bound on singular value deviations,
\begin{align}
\frac{1}{d}\sum_{i=1}^d\Bigl|\trunc_R\bigl(s_i(X_d)\bigr)-\trunc_R\bigl(s_i(Y_d)\bigr)\Bigr|
\le L_R\,\|X_d-Y_d\|_\infty\to 0.
\end{align}

\textbf{Step 3 (tail bound).}
By Markov's inequality and Step 1,
\begin{align}
\frac{1}{d}\sum_{i:\,s_i(X_d)\ge R}s_i(X_d)^{2q}\le \frac{1}{R^{2}}\cdot \frac{1}{d}\sum_{i=1}^d s_i(X_d)^{2q+2}
= \frac{m_{q+1}(X_dX_d^\top)}{R^{2}}
\le \frac{C_{q+1}}{R^{2}}.
\end{align}
The same bound holds with $X_d$ replaced by $Y_d$ since $m_{q+1}(Y_dY_d^\top)$ converges.
Therefore,
\begin{align}
\limsup_{d\to\infty}\Bigl|m_q(X_dX_d^\top)-m_q(Y_dY_d^\top)\Bigr|
\le \limsup_{d\to\infty}\frac{1}{d}\sum_{i=1}^d\Bigl|\trunc_R\bigl(s_i(X_d)\bigr)-\trunc_R\bigl(s_i(Y_d)\bigr)\Bigr|
+ \frac{2C_{q+1}}{R^{2}}.
\end{align}
The first term tends to $0$ by Step 2. Letting $R\to\infty$ yields
\begin{align}
m_q(X_dX_d^\top)-m_q(Y_dY_d^\top)\to 0.
\end{align}
Since $m_q(Y_dY_d^\top)\to m_q(\nu)$, we conclude $m_q(X_dX_d^\top)\to m_q(\nu)$.
\end{proof}

\section{Settings for Numerical Simulations}\label{sec:numerical-settings}
All experiments were conducted in Python 3.10.12 on a Linux system.  Unless otherwise noted, the plotted curves report the mean over 10 independent random trials, and error bars indicate one standard deviation; in several figures, the error bars are visually negligible due to their small magnitude.

\end{document}